\documentclass[10pt]{article}
\PassOptionsToPackage{numbers, sort&compress, comma, square}{natbib}
\usepackage{amsfonts}
\usepackage{amsmath}
\allowdisplaybreaks
\usepackage{siunitx}
\usepackage{xfrac}
\usepackage{bm}
\usepackage{mathtools}
\usepackage{mathabx}
\usepackage{nicefrac}
\usepackage[accepted]{tmlr}
\usepackage[english]{babel}
\usepackage{xspace}
\usepackage[utf8]{inputenc}
\usepackage[T1]{fontenc}
\usepackage{microtype}
\usepackage{placeins}
\usepackage{titletoc}
\usepackage{xcolor}
\usepackage{graphicx}
\usepackage{subfigure}
\usepackage[labelfont=bf]{caption}
\usepackage{wrapfig}
\usepackage{enumitem}
\usepackage{booktabs}
\usepackage{multirow}
\usepackage{soul}
\usepackage{changepage,threeparttable}
\usepackage{tabularx}
\usepackage{multirow}
\usepackage{etoolbox}
\newrobustcmd{\B}{\bfseries}

\definecolor{mydarkblue}{rgb}{0,0.08,0.45} 
\usepackage[
    colorlinks=true,
    linkcolor=mydarkblue,
    citecolor=mydarkblue,
    filecolor=mydarkblue,
    urlcolor=mydarkblue,
]{hyperref}
\usepackage[nameinlink,capitalize,noabbrev]{cleveref}
\usepackage{url}
\usepackage[numbered]{bookmark}

\usepackage{tikz}

\DeclareRobustCommand{\colordot}[1]{\begin{tikzpicture}[baseline=(a.south)]
    \node[circle, scale=0.75,color=white, fill=#1] (a) {};
  \end{tikzpicture}}

\DeclareRobustCommand{\colorcircle}[1]{\begin{tikzpicture}[baseline=(a.south)]
    \node[circle, scale=0.65,fill=white, draw=#1] (a) {};
  \end{tikzpicture}}

\DeclareRobustCommand{\dashedcolorline}[1]{\begin{tikzpicture}
    \raisebox{1.5pt}{
      \draw[#1,dashed,line width=1.5pt] (0,0) -- (1em,0);
    }
  \end{tikzpicture}}

\usepackage{pifont}

\DeclareRobustCommand{\colorgradientbox}[2][white]{%
  \begin{tikzpicture}[baseline=-0.5ex]%
    \node [rectangle, left color=#1, right color=#2, anchor=base, minimum width=1.75em, minimum height=1em, draw=black] (box) at (0,0){};%
  \end{tikzpicture}%
}

\newcommand*{\eg}{e.g.\@\xspace}
\newcommand*{\ie}{i.e.\@\xspace}

\usepackage[super]{nth}

\newcommand{\dataset}{dataset\xspace}
\newcommand{\datasets}{datasets\xspace}
\newcommand{\datapoint}{data point\xspace}
\newcommand{\datapoints}{data points\xspace}
\newcommand{\runtime}{runtime\xspace}

\newcommand{\adam}{\textsc{\mbox{Adam}}\xspace}

\newcommand{\itergp}{\textsc{IterGP}\xspace}
\newcommand{\iterncgp}{\textsc{IterNCGP}\xspace}
\newcommand{\svgp}{\textsc{SVGP}\xspace}

\newcommand{\mnist}{\textsc{MNIST}\xspace}

\newcommand{\keops}{\textsc{KeOps}\xspace}
\newcommand{\gpytorch}{\textsc{GPyTorch}\xspace}

\newcommand{\pytorch}{\textsc{PyTorch}\xspace}

\DeclareMathOperator*{\argmax}{arg\,max}

\DeclareMathOperator{\Tr}{Tr}
\DeclareMathOperator{\diag}{diag}

\DeclareMathOperator{\vspan}{span}

\newcommand{\defeq}{\coloneqq}
\newcommand{\eqdef}{\eqqcolon}
\newcommand{\eqc}{\overset{\text{c}}{=}}

\newcommand{\R}{\mathbb{R}}
\newcommand{\N}{\mathbb{N}}

\newcommand*{\spacesym}[1]{{\mathbb{#1}}}

\newcommand*{\pinv}{\dagger}
\newcommand*{\blockdiag}{\operatorname{blockdiag}}

\DeclarePairedDelimiterXPP\bigO[1]{\mathcal{O}}{(}{)}{}{#1}
\DeclarePairedDelimiterXPP\smallo[1]{o}{(}{)}{}{#1}
\DeclarePairedDelimiterXPP\bigOmega[1]{\Omega}{(}{)}{}{#1}
\DeclarePairedDelimiterXPP\smallomega[1]{\omega}{(}{)}{}{#1}
\DeclarePairedDelimiterXPP\bigTheta[1]{\Theta}{(}{)}{}{#1}

\newcommand*{\inputdim}{D}

\newcommand*{\inputspace}{\spacesym{X}}
\newcommand*{\outputspace}{\spacesym{Y}}
\newcommand*{\numclasses}{C}
\newcommand*{\numtraindata}{N}
\newcommand*{\traindata}{\mX}

\newcommand*{\symboltestdata}{\diamond}
\newcommand*{\numtestdata}{N_\symboltestdata}

\newcommand*{\buffersize}{B}
\newcommand*{\bufferlimit}{R}
\newcommand*{\matrixroot}{\mQ}

\newcommand*{\timeK}{\tau_{\mK}}
\newcommand*{\spaceK}{\mu_{\mK}}
\newcommand*{\timeWinv}{\tau_{\mW^{-1}}}
\newcommand*{\spaceWinv}{\mu_{\mW^{-1}}}
\newcommand*{\timepolicy}{\tau_{\textsc{policy}}}
\newcommand*{\timemean}{\tau_{\vm}}

\newcommand*{\newtonidx}{i}
\newcommand*{\solveridx}{j}
\newcommand*{\totalsolveriters}{J}

\newcommand*{\cnordering}{$\numclasses\numtraindata$-ordering\xspace}
\newcommand*{\ncordering}{$\numtraindata\numclasses$-ordering\xspace}

\newcommand*{\rweightsapprox}{\vv}

\newcommand*{\kernmatinvapprox}{\mC}

\newcommand*{\gp}[2]{{\ensuremath{\operatorname{\mathcal{GP}}}\left(#1, #2\right)}}

\newcommand*{\kernel}{K}
\newcommand*{\kernelfn}{K}

\newcommand*{\matern}[1]{Matérn$(\frac{#1}{2})$}

\newcommand*{\action}{\vs}

\newcommand*{\observ}{\alpha}
\newcommand*{\searchdir}{\vd}

\newcommand*{\searchdirsqnorm}{\eta}

\newcommand*{\residual}{\vr}

\newcommand{\KWinv}{\hat{\mK}}
\newcommand{\KWinvf}[1]{\hat{\mK}(#1)}

\newcommand{\pseudotargets}{\hat{\vy}}
\newcommand{\pseudotargetsf}[1]{\hat{\vy}(#1)}

\newcommand{\convtol}{\delta}

\newcommand*{\gaussian}[2]{{\ensuremath{\operatorname{\mathcal{N}}\mathopen{}\left(#1, #2\right)}}}
\newcommand*{\gaussianpdf}[3]{{\ensuremath{\operatorname{\mathcal{N}}\mathopen{}\left(#1; #2, #3\right)}}}

\def\vzero{{\bm{0}}}
\def\vone{{\bm{1}}}

\def\vlambda{{\bm{\lambda}}}
\def\vpi{{\bm{\pi}}}

\def\vb{{\bm{b}}}

\def\vd{{\bm{d}}}
\def\ve{{\bm{e}}}
\def\vf{{\bm{f}}}
\def\vg{{\bm{g}}}

\def\vm{{\bm{m}}}

\def\vr{{\bm{r}}}
\def\vs{{\bm{s}}}
\def\vt{{\bm{t}}}
\def\vu{{\bm{u}}}
\def\vv{{\bm{v}}}
\def\vw{{\bm{w}}}
\def\vx{{\bm{x}}}
\def\vy{{\bm{y}}}
\def\vz{{\bm{z}}}

\def\mzero{{\bm{0}}}
\def\mA{{\bm{A}}}

\def\mC{{\bm{C}}}

\def\mI{{\bm{I}}}

\def\mK{{\bm{K}}}

\def\mM{{\bm{M}}}

\def\mP{{\bm{P}}}
\def\mQ{{\bm{Q}}}

\def\mS{{\bm{S}}}
\def\mT{{\bm{T}}}
\def\mU{{\bm{U}}}

\def\mW{{\bm{W}}}
\def\mX{{\bm{X}}}

\def\mLambda{{\bm{\Lambda}}}
\def\mSigma{{\bm{\Sigma}}}
\def\mOmega{{\bm{\Omega}}}

\def\mPi{{\bm{\Pi}}}

\DeclareMathAlphabet{\mathsfit}{\encodingdefault}{\sfdefault}{m}{sl}
\SetMathAlphabet{\mathsfit}{bold}{\encodingdefault}{\sfdefault}{bx}{n}

\def\sL{{\mathbb{L}}}

\usepackage{amsthm}
\usepackage{thmtools}
\usepackage{thm-restate}

\declaretheoremstyle[
headfont=\normalfont\bfseries,
notefont=\normalfont,
bodyfont=\normalfont,
headpunct={},
postheadspace=\newline,
spaceabove=1.5\parskip,  ]{definitionstyle}

\declaretheoremstyle[
headfont=\normalfont\bfseries,
notefont=\normalfont,
bodyfont=\normalfont\itshape,
headpunct={},
postheadspace=\newline,
spaceabove=1.5\parskip,  ]{lemmastyle}

\declaretheoremstyle[
headfont=\normalfont\bfseries,
notefont=\normalfont,
bodyfont=\normalfont\itshape,
headpunct={},
postheadspace=\newline,
spaceabove=1.5\parskip,  ]{theoremstyle}

\declaretheoremstyle[
headfont=\normalfont\bfseries,
notefont=\normalfont,
bodyfont=\normalfont,
headpunct={},
spaceabove=1.5\parskip,  ]{remarkstyle}

\declaretheorem[style=definitionstyle,name=Definition,numberwithin=section]{definition}
\declaretheorem[style=lemmastyle,name=Lemma,sibling=definition]{lemma}
\declaretheorem[style=lemmastyle,name=Proposition,sibling=definition]{proposition}

\declaretheorem[style=theoremstyle,name=Theorem,sibling=definition]{theorem}

\usepackage{scrhack}
\usepackage{algorithm}
\usepackage[noend]{algpseudocode}

\algrenewcommand{\algorithmiccomment}[3]{\hfill {\small \textcolor{darkgray}{#1 \hspace{0.5em} \makebox[5.5em][l]{#2} \makebox[3.5em][l]{#3}}}}
\algrenewcommand\algorithmicindent{1em}
\algrenewcommand\alglinenumber[1]{\small {\textcolor{darkgray}{#1}}}

\makeatletter
\expandafter\patchcmd\csname\string\algorithmic\endcsname{\itemsep\z@}{\itemsep=0.25ex}{}{}
\newcommand\fs@booktabsruled{\def\@fs@cfont{\bfseries\strut}\let\@fs@capt\floatc@ruled
    \def\@fs@pre{\hrule height\heavyrulewidth depth0pt \kern\belowrulesep}\def\@fs@mid{\kern\aboverulesep\hrule height\lightrulewidth\kern\belowrulesep}\def\@fs@post{\kern\aboverulesep\hrule height\heavyrulewidth\relax}\let\@fs@iftopcapt\iftrue
}
\makeatother
\floatstyle{booktabsruled}
\restylefloat{algorithm}

\definecolor{Recycling}{HTML}{4169E1}
\definecolor{TUred}{RGB}{165,30,55}
\definecolor{TUdark}{RGB}{50,65,75}
\definecolor{TUgold}{RGB}{180,160,105}
\definecolor{TUgray}{RGB}{185,184,188}
\definecolor{TUdarkblue}{RGB}{65,90,140}
\definecolor{TUblue}{RGB}{0,105,170}
\definecolor{TUlightblue}{RGB}{80,170,200}
\definecolor{TUlightgreen}{RGB}{125,165,75}
\definecolor{TUgreen}{RGB}{125,165,75}
\definecolor{TUdarkgreen}{RGB}{50,110,30}
\definecolor{TUlightred}{RGB}{200,80,60}
\definecolor{TUpurple}{RGB}{175,110,150}
\definecolor{TUorange}{RGB}{210,150,0}
\definecolor{SNSblue}{rgb}{0.1216, 0.4666, 0.7059}
\definecolor{SNSorange}{rgb}{1.0, 0.4980, 0.0549}
\definecolor{SNSgreen}{rgb}{0.1725, 0.6274, 0.1725}
\definecolor{SNSred}{rgb}{0.84, 0.15, 0.16}
\definecolor{SNSpurple}{rgb}{0.58, 0.40, 0.74}
\definecolor{SNSorange_shaded}{HTML}{ffcea3}
\definecolor{SNSblue_shaded}{HTML}{8ebad9}
\definecolor{SNSgreen_shaded}{HTML}{cae7ca}
\definecolor{SNSred_shaded}{HTML}{ea9293}
\definecolor{MPLred_shaded}{HTML}{df735b}
\definecolor{MPLblue_shaded}{HTML}{3885bc}
\definecolor{PlotRed}{rgb}{0.77, 0.31, 0.32}
\definecolor{PlotBlue}{rgb}{0.30, 0.45, 0.69}

\newcommand{\papertitle}{Accelerating Non-Conjugate Gaussian Processes\\By Trading Off Computation For Uncertainty}

\title{\papertitle}

\author{\name Lukas Tatzel
  \email lukas.tatzel@uni-tuebingen.de \\
  \addr University of Tübingen, Tübingen AI Center
  \AND
  \name Jonathan Wenger
  \email jw4246@columbia.edu \\
  \addr Columbia University
  \AND
  \name Frank Schneider
  \email f.schneider@uni-tuebingen.de\\
  \addr University of Tübingen, Tübingen AI Center
  \AND
  \name Philipp Hennig
  \email philipp.hennig@uni-tuebingen.de\\
  \addr University of Tübingen, Tübingen AI Center
}

\def\month{04}
\def\year{2025}
\def\openreview{\url{https://openreview.net/forum?id=UdcF3JbSKb}}

\begin{document}

\maketitle

\begin{abstract}
Non-conjugate Gaussian processes (NCGPs) define a flexible probabilistic framework to
model categorical, ordinal and continuous data, and are widely used in practice.
However, exact inference in NCGPs is prohibitively expensive for large datasets, thus
requiring approximations in practice. The approximation error adversely impacts the
reliability of the model and is not accounted for in the uncertainty of the prediction.
We introduce a family of iterative methods that explicitly model this error. They are
uniquely suited to parallel modern computing hardware, efficiently recycle computations,
and compress information to reduce both the time and memory requirements for NCGPs. As
we demonstrate on large-scale classification problems, our method significantly
accelerates posterior inference compared to competitive baselines by trading off reduced
computation for increased uncertainty.
\end{abstract}

\section{Introduction}
\label{sec:intro}

Non-conjugate Gaussian processes\footnote{Such a model is also called Generalized
Gaussian Process Model \citep{Chan2011Generalized} or Generalized Linear Model
\citep{Nelder1972GeneralizedLinear}. The latter name is sometimes used only for latent
Gaussian models, which can lead to confusion. The models studied in this work are
generally of \textit{nonparametric} nature. The resulting large linear problems are the
main reason why the algorithms we propose are relevant in the first place.}  (NCGPs)
form a fundamental interpretable model class widely used throughout the natural and
social sciences. For example, NCGPs are applied to count data in biomedicine,
categorical data in object classification tasks, and continuous data in time series
regression. An NCGP assumes the data is generated from an exponential family likelihood
with a Gaussian process (GP) prior over the latent function. Such a
\textit{probabilistic} approach is essential in domains where critical decisions must be
made based on limited information, such as in public policy, medicine or robotics.

Unfortunately, even the conjugate Gaussian case, where fitting an NCGPs reduces to GP
regression, naively has cubic time complexity \(\bigO{\numtraindata^3}\) in the number
of training data \(\numtraindata\) and requires \(\bigO{\numtraindata^2}\) memory, which
is prohibitive for modern large-scale datasets. For non-Gaussian likelihoods, inference
has to be done approximately, which generally exacerbates this problem. For example,
inference via the Laplace approximation (LA) boils down to finding the mode of the log
posterior via Newton's method, which is equivalent to solving a \emph{sequence} of
regression problems
\citep{Spiegelhalter1990SequentialUpdating,MacKay1992EvidenceFramework,Bishop2006}.

Due to limited computational resources, large-scale problems often require
approximations. The resulting error affects a model's predictive accuracy but also its
uncertainty quantification. Hence, the question arises: \textbf{Can NCGPs be efficiently
trained on extensive data without compromising reliability?}

\vspace{-0.3ex}

\begin{figure}
  \centering
  \includegraphics[width=0.99\linewidth]{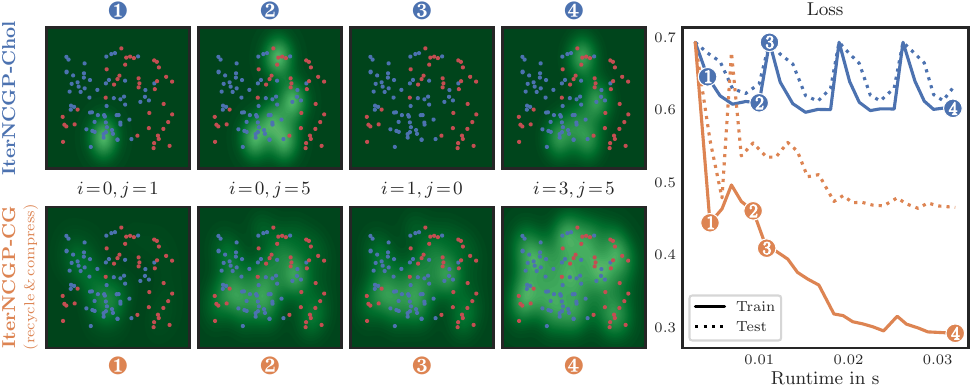}
  \caption{
    \textbf{Binary Classification with \iterncgp.}
    Comparison of two \iterncgp variants: \textcolor{SNSblue}{\textit{(Top)}} \iterncgp
    variant corresponding to data subsampling and solving each regression problem
    exactly in each Newton step \(\newtonidx\). \textcolor{SNSorange}{\textit{(Bottom)}}
    \iterncgp variant with a more informative policy (details in \cref{sec:policy}),
    recycling of computations between Newton steps (details in \cref{sec:recycling}) and
    compression to reduce memory (details in \cref{sec:compression}). The panels show
    the marginal uncertainty (\,\colorgradientbox[black!70!SNSgreen]{SNSgreen}\,) over
    the latent function at Newton step $i$ and solver iteration $j$. Using recycling,
    the current belief is efficiently propagated between mode-finding steps
    \(\newtonidx\) (\ding{183} \(\!\to\!\) \ding{184}) without performance drops
    \textit{(Right)}. Details in \cref{sec:details_binary_classification}.}
    \label{fig:visual_abstract}
    \vspace{-0.3ex}
\end{figure}

Recently, iterative methods have emerged which in the conjugate Gaussian case allow an
explicit, tunable trade-off between reduced computation and increased uncertainty
\citep{Trippe2019LRGLMHighDimensional,Wenger2022PosteriorComputational}. This
computational uncertainty quantifies the inevitable approximation error in the sense of
probabilistic numerics \citep{Hennig2015a,Cockayne2019a,Oates2019,Hennig2022}.

\vspace{-0.3ex}

\textbf{Contributions.} In this work, we take a similar approach and extend
\citet{Wenger2022PosteriorComputational}'s \itergp (that assumes a \textit{conjugate}
Gaussian likelihood) to \textit{non-conjugate} exponential family likelihoods. This is a
non-trivial extension, as the posterior is no longer Gaussian and \textit{multiple}
related regression problems have to be solved. Specifically, we propose \textbf{(i)}
\iterncgp{}: a family of efficient inference algorithms for NCGPs with a tunable
trade-off between computational savings and added uncertainty
(\cref{sec:derivation_iterglm}) with \textbf{(ii)} mechanisms to tailor the inference
algorithm to a specific downstream application (\cref{sec:policy}). In response to the
specific computational challenges in the non-conjugate setting, we develop
\textbf{(iii)} novel strategies to optimally recycle costly computations
(\cref{sec:recycling}) and \textbf{(iv)} to restrict the memory usage, with minimal
impact on inference (\cref{sec:compression}).

\vspace{-0.3ex}

Our algorithm \iterncgp consists of two nested loops: An outer loop (indexed by
$\newtonidx$) which iterates over Newton steps\slash{}GP regression problems, each of
which is solved approximately via an inner loop (indexed by $\solveridx$) that
implements a probabilistic linear solver. \Cref{fig:visual_abstract} shows the marginal
uncertainty over the latent function at different stages of that process and illustrates
the effectiveness of (ii), (iii) and (iv). Specifically, by recycling computations
between Newton steps, we are able to traverse the two-loop structure ``diagonally''
which leads to steady progress without performance drops between Newton steps.

\vspace{-0.3ex}

\section{Background}
\label{sec:background}

Let \((\mX, \vy)\) be a \dataset of $\numtraindata$ input vectors $\{\vx_n \in
\inputspace\}_{n=1}^\numtraindata$ stacked into $\mX  = (\vx_1, \dots\vx_N)^\top \in
\R^{\numtraindata \times \inputdim}$ and corresponding outputs $\vy = (y_1, \dots,
y_N)^\top \in \outputspace^N$, where $\inputspace = \R^\inputdim$ and $\outputspace =
\R$ or \(\outputspace  =  \N_0\) (regression) or $\outputspace = \{1, \dots, C\}$
(classification).

\subsection{Non-conjugate Gaussian Processes (NCGPs)}
\label{sec:ncgps}
We consider the probabilistic model $p(\vy, \vf \mid \mX) = p(\vy \mid \vf) \, p(\vf
\mid \mX)$, where the vector \(\vf \defeq f(\mX) \in \R^{\numtraindata\numclasses}\) is
given by a latent function $f \colon \inputspace \to \R^\numclasses$ evaluated at the
training data.

\textbf{Prior.}
Assume a multi-output Gaussian process prior \(\gp{m}{\kernelfn}\) over the latent
function with mean function $m \colon \inputspace \to  \R^\numclasses$ and kernel
function $\kernelfn \colon \inputspace \times \inputspace \to \R^{\numclasses \times
\numclasses}$. Therefore the latent vector has density $p(\vf\mid\mX) =
\gaussianpdf{\vf}{\vm}{\mK}$ with mean $\vm \defeq m(\traindata) \in
\R^{\numtraindata\numclasses}$ and covariance $\mK = \kernelfn(\traindata,\traindata)
\in \R^{\numtraindata\numclasses \times \numtraindata\numclasses}$ defined by
$\numtraindata^2$ blocks $\kernelfn(\vx_i, \vx_j) \in \R^{\numclasses \times
\numclasses}$. Each such block represents the covariance between the $\numclasses$
latent functions evaluated at inputs $\vx_i$ and $\vx_j$.

\textbf{Likelihood.}
Assume iid data, which depends on the latent function via an inverse link function
\(\lambda : \R^\numclasses \to \R^\numclasses\), s.t. $p(\vy \mid \vf) =
\prod_{n=1}^\numtraindata p(y_n \mid \lambda(\vf_n))$, where \(p(y_n \mid
\lambda(\vf_n))\) is a log-concave likelihood, \eg any exponential family
distribution.\footnote{The Hessian of an exponential family likelihood is the negative
Hessian of its log-partition function, which equals the \emph{positive definite}
covariance matrix of its sufficient statistics.} For example, for Poisson regression the
inverse link function is given by \(\lambda(f_n) = \exp(f_n)\) and for multi-class
classification by \(\lambda(\vf_n) = \operatorname{softmax}(\vf_n)\).

For nonlinear inverse link functions, the posterior $p(f  \mid  \mX, \vy)$ and
predictive distribution \(p(y_\symboltestdata  \mid  \mX, \vy, \vx_\symboltestdata) =
\int p(y_\symboltestdata  \mid f_\symboltestdata) p(f_\symboltestdata  \mid  \mX, \vy,
\vx_\symboltestdata) df_\symboltestdata \) are computationally intractable, requiring
approximations.

\subsection{Approximate Inference via Laplace}
\label{subsec:laplace_approximation}

A popular way to perform approximate inference in an NCGP is to use a Laplace
approximation (LA) \citep{Spiegelhalter1990SequentialUpdating,
MacKay1992EvidenceFramework, Bishop2006}. The idea is to approximate the posterior
\begin{equation}
    p(\vf \mid \mX, \vy)
    \approx
    q(\vf \mid \mX, \vy)
    \defeq
    \gaussianpdf{\vf}{\vf_{\text{MAP}}}{\mSigma},
    \label{eq:laplace_approximation}
\end{equation}
with a Gaussian with mean given by the mode \(\vf_{\text{MAP}}\) of the log-posterior
and covariance $\mSigma \defeq -(\nabla^2 \log p(\vf_{\text{MAP}} \mid \mX, \vy))^{-1}$
given by the negative inverse Hessian (with respect to $\vf$) at the mode. Due to the
assumed GP prior over the latent function, the log-posterior is given by
\begin{equation}
    \label{eq:Psi}
    \Psi(\vf)
    \defeq \log p(\vf \mid \mX, \vy)
    \eqc \log p(\vy \mid \vf) + \log p(\vf \mid \mX)
    \eqc \log p(\vy \mid \vf)
    -\frac{1}{2} (\vf - \vm)^\top \mK^{-1} (\vf - \vm)
\end{equation}
We use $\eqc$ to denote equality up to an additive constant.

\textbf{Mode-Finding via Newton's Method.}
To find the mode $\vf_{\text{MAP}}$, one typically uses Newton steps, \ie
\begin{equation}
    \vf_{\text{MAP}}
    \approx \vf_{\newtonidx + 1}
    = \vf_\newtonidx
    - \nabla^2 \Psi(\vf_\newtonidx)^{-1} \cdot \nabla\Psi(\vf_\newtonidx),
    \label{eq:newton_step}
\end{equation}
where
$
    \nabla \Psi(\vf_\newtonidx)
    = \nabla \log p(\vy \mid \vf_\newtonidx) - \mK^{-1} (\vf_\newtonidx - \vm)
$ and
$
    \nabla^2 \Psi(\vf_\newtonidx)
    = -\mW(\vf_\newtonidx) - \mK^{-1}
$.
The negative Hessian $\mW(\vf_\newtonidx) \defeq -\nabla^2 \log p(\vy \mid
\vf_\newtonidx)$ of the log likelihood at $\vf_\newtonidx$ is positive definite for all
\(\vf\), since we assumed a log-concave likelihood. Therefore \(\Psi\) is concave and
the Newton updates are well-defined.

\subsection{Predictions}
\label{subsec:prediction}

Using a local quadratic Taylor approximation of the log-posterior $\Psi$ around the
current iterate \(\vf_\newtonidx\), we obtain the LA $\gaussianpdf{\vf}{\vf_{\newtonidx
+ 1}}{-\nabla^2 \Psi(\vf_\newtonidx)^{-1}}$ whose mean is given by the maximizer of the
local quadratic, \ie the subsequent Newton iterate \(\vf_{\newtonidx + 1}\).
Substituting this in place of the posterior, the predictive distribution for the latent
function \(p(f(\cdot) \mid \mX, \vy)=\int p(f(\cdot) \mid \vf) \, \mathcal{N}(\vf;
\vf_{\newtonidx + 1}, -\nabla^2 \Psi(\vf_\newtonidx)^{-1}) \, d\vf\) is a Gaussian
process $\gp{m_{\newtonidx,*}}{\kernelfn_{\newtonidx,*}}$, with mean and covariance
functions
\begin{align}
    m_{\newtonidx,*}(\cdot)
    &\defeq m(\cdot)
    + \kernelfn(\cdot, \mX) \mK^{-1} (\vf_{\newtonidx + 1} - \vm),
    \label{eq:predictive_mean}
    \\
    \kernelfn_{\newtonidx,*}(\cdot, \cdot)
    &\defeq \kernelfn(\cdot, \cdot)
    - \kernelfn(\cdot, \mX)
    \hat{\mK}(\vf_\newtonidx)^{-1}
    \kernelfn(\mX, \cdot),
    \label{eq:predictive_variance}
\end{align}
where $\KWinvf{\vf_\newtonidx} \defeq \mK + \mW(\vf_\newtonidx)^{-1}$
\citep[cf.~Eq.~(3.24);][]{Rasmussen2006}. We obtain the predictive distribution for
$y_\symboltestdata$ at test input \(\vx_\symboltestdata\) by integrating this
approximative posterior against the likelihood, \ie
$p(y_\symboltestdata \mid \mX, \vy, \vx_\symboltestdata)
     =
    \int
    p(y_\symboltestdata \mid \vf_\symboltestdata) \,
    p(\vf_\symboltestdata \mid \mX, \vy, \vx_\symboltestdata) \,
    d\vf_\symboltestdata$.
This $\numclasses$-dimensional integral can be approximated via quadrature, MC-sampling
or specialized approaches (like the probit method \citep{MacKay1992EvidenceFramework}
for a categorical likelihood and softmax inverse link function).

\section{Computation-Aware Inference in NCGPs}
\label{sec:method}

While Newton's method typically converges in a few steps for a log-concave likelihood,
each step in \eqref{eq:newton_step} requires linear system solves with symmetric
positive \mbox{(semi-)}definite matrices of size \(\numtraindata\numclasses   \times
\numtraindata\numclasses\). Naively computing these solves via Cholesky decomposition is
problematic even for moderately sized \datasets due to its cubic time
$\bigO{\numtraindata^3\numclasses^3}$ and quadratic memory complexity
\(\bigO{\numtraindata^2\numclasses^2}\). We will demonstrate in the following how to
circumvent this issue by reducing the computations in exchange for increased uncertainty
about the latent function.

\subsection{Derivation of the \iterncgp Framework}
\label{sec:derivation_iterglm}

\textbf{Overview.} As a first step, we reinterpret the posterior predictive mean
(\cref{eq:predictive_mean}) as the GP posterior for a specific regression problem
(\cref{eq:inference_between_steps_mean_2}). The sequence of regression problems is
connected to the sequence of Newton steps and forms the outer loop of our algorithm
\iterncgp, indexed by $\newtonidx$ (\cref{alg:outer}). The formulation as GP regression
problem allows us to apply \itergp \citep{Wenger2022PosteriorComputational} as an inner
loop, indexed by $\solveridx$ (\cref{alg:inner}). By using \itergp, we can solve each
regression problem approximately and quantify the resulting error in the form of
additional uncertainty (\cref{eq:inference_PLS_iteration_covar}).

\textbf{Outer Loop: Newton's Method as Sequential GP Regression.}
Through the LA, we obtain a posterior predictive $f \sim
\gp{m_{\newtonidx,*}}{\kernelfn_{\newtonidx,*}}$ over the latent function for each
Newton step. As we show in \cref{subsec:connection_inference_Newton_steps}, the
posterior predictive (\cref{eq:predictive_mean,eq:predictive_variance}) in step
\(\newtonidx\) can be written as
\begin{align}
    m_{\newtonidx,*}(\cdot)
  & =  m(\cdot)
    + \kernelfn(\cdot, \mX)
  \KWinvf{\vf_\newtonidx}^{-1}
  (\pseudotargetsf{\vf_\newtonidx} - \vm)
  \label{eq:inference_between_steps_mean_2} \\
    \kernelfn_{\newtonidx,*}(\cdot, \cdot)
    & =  \kernelfn(\cdot, \cdot)
    - \kernelfn(\cdot, \mX)
  \KWinvf{\vf_\newtonidx}^{-1}
    \kernelfn(\mX, \cdot),
  \label{eq:inference_between_steps_covar}
\end{align}
where $\smash{\pseudotargetsf{\vf_\newtonidx} \defeq \vf_\newtonidx +
\mW(\vf_\newtonidx)^{-1} \nabla\log p(\vy \mid \vf_\newtonidx)}$.
\Cref{eq:inference_between_steps_mean_2,eq:inference_between_steps_covar} have the exact
form of the posterior for a GP regression problem with fictitious \emph{pseudo targets}
$\pseudotargetsf{\vf_\newtonidx}$ observed with Gaussian noise
$\gaussian{\vzero}{\mW(\vf_\newtonidx)^{-1}}$.\footnote{If \(\mW(\vf_\newtonidx)^{-1}\)
does not exist, \eg in multi-class classification, we substitute its pseudo-inverse
\(\mW(\vf_\newtonidx)^\pinv\), which for multi-class classification can be evaluated
efficiently (see \Cref{subsec:pseudo_inverse_w}). Alternatively, one can place a prior
on the sum of the \(\numclasses\) latent functions
\citep[Eq.~(10)]{MacKay1998ChoiceBasis}.}
\cref{fig:pseudo_targets_noise} shows an illustration of this interpretation.

\begin{figure}[tb]
  \centering
  \includegraphics[width=0.99\linewidth]{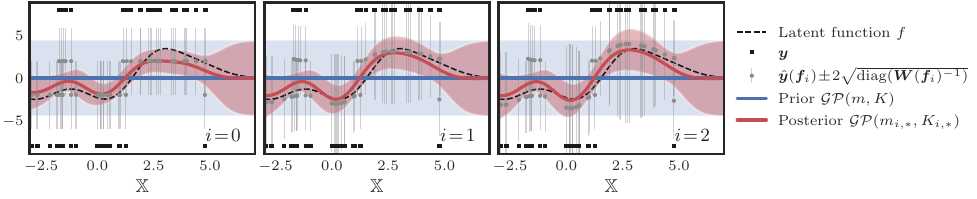}
  \caption{\textbf{Approximate Inference in NCGPs as Sequential GP Regression.}
  Performing a LA at a Newton iterate $\vf_\newtonidx$ results in a posterior GP that
  coincides with the posterior to a GP regression problem with pseudo targets
  $\pseudotargetsf{\vf_\newtonidx}$ observed with Gaussian noise
  \(\gaussian{\vzero}{\mW(\vf_\newtonidx)^{-1}}\). The plot shows an illustration of
  this connection for binary classification on a toy problem with the latent function
  drawn from a GP. Notice how similar the posteriors are between Newton steps. This
  motivates our proposed strategy for recycling computations between steps in
  \cref{sec:recycling}. Details in \cref{sec:details_binary_classification}. }
  \label{fig:pseudo_targets_noise}
\end{figure}

\Cref{eq:inference_between_steps_mean_2} requires solving a linear system
$\KWinvf{\vf_\newtonidx} \, \vv = \pseudotargetsf{\vf_\newtonidx} - \vm$ of size
\(\numtraindata\numclasses \times \numtraindata\numclasses\). The posterior mean is then
simply given by $m_{\newtonidx,*}(\cdot) = m(\cdot) + \kernelfn(\cdot, \mX) \vv$.
However, also the Newton update from \cref{eq:newton_step} follows directly from $\vv$
since $\vf_{\newtonidx+1} = \mK \vv + \vm$. In that sense, computing the posterior
predictive mean and performing Newton updates are equivalent.
The sequence of Newton steps\slash{}GP regression problems forms the outer loop of our
algorithm \iterncgp. The pseudo code is given in \cref{alg:outer}.

\begin{algorithm*}[t!]  \caption{\textbf{\iterncgp Outer loop\label{alg:outer}.}}
  \small
  \algrenewcommand{\algorithmiccomment}[3]
{\hfill
    {
        \small \textcolor{gray}{#1 \hspace{0.3em}}
        \makebox[7em][l]{#2}   \makebox[5em][l]{#3}  }
}

\textbf{Input:} GP prior $\gp{m}{\kernel}$, training data $(\mX, \vy)$, $\nabla p(\vy
\vert \vf, \mX)$ and access to products with $\mK$ and $\mW(\vf)^{-1}$\\
\textbf{Output:} GP posterior
$\gp{m_{\newtonidx,\solveridx}}{\kernelfn_{\newtonidx,\solveridx}}$
\begin{algorithmic}[1]
    \Procedure{\iterncgp}{$m, \kernelfn, \mX, \vy, \vf_0 = \vm$}
    \Comment{}{\textbf{Time}}{\textbf{Memory}}
    \State $\vm \gets m(\mX)$,
    \Comment{$\vm$: Prior mean vector}
    {$\bigO{\timemean}$}
    {$\bigO{\numtraindata\numclasses}$}
    \State Provide access to $\vw \mapsto \mK \vw$
    \Comment{$\mK$: Prior covariance/kernel matrix}{}{$\bigO{\spaceK}$}
    \State \textcolor{Recycling}{
        Initialize buffers $\mS, \mT \in \R^{\numtraindata\numclasses \times 0}$
    }
    \Comment{$\mS$: actions, $\mT$: products with $\mK$}{}{}
    \For{$\newtonidx = 0, 1, 2, \dots$ \textbf{while not} \textsc{OuterStoppingCriterion}()}
        \State Provide access to $\vw \mapsto \mW(\vf_\newtonidx)^{-1} \vw$
        \Comment{$\mW(\vf_\newtonidx)^{-1}$: Observation noise}{}{$\bigO{\spaceWinv}$}
        \State
        $
            \pseudotargetsf{\vf_\newtonidx} \gets
            \vf_\newtonidx
            + \mW(\vf_\newtonidx)^{-1} \nabla\log p(\vy \vert \vf_\newtonidx)
        $
        \Comment{$\pseudotargetsf{\vf_\newtonidx}$: Pseudo targets}
        {$\bigO{\timeWinv\!+\!\numtraindata\numclasses}$}
        {$\bigO{\numtraindata\numclasses}$}
        \State
        $\mathcal{GP}(m_{\newtonidx, \solveridx}, \kernel_{\newtonidx, \solveridx}), \vv_\solveridx \gets$
        \itergp{}($m, \kernelfn, \mX, \vy, \vm, \mK, \mW(\vf_\newtonidx)^{-1}, \pseudotargetsf{\vf_\newtonidx}, \textcolor{Recycling}{\mS}, \textcolor{Recycling}{\mT}$)
        \Comment{}{}{}
        \State
        $\vf_{i+1} \gets \mK \vv_\solveridx + \vm$
        \Comment{Approximate Newton update}{$\bigO{\timeK\!+\!\numtraindata\numclasses}$}
        {$\bigO{\numtraindata\numclasses}$}
    \EndFor
    \State \Return $\gp{m_{\newtonidx,\solveridx}}{\kernelfn_{\newtonidx,\solveridx}}$
    \EndProcedure
\end{algorithmic}

The \itergp algorithm is given in \cref{alg:inner}. Instructions in
\textcolor{Recycling}{blue} are needed for recycling (see \cref{sec:recycling}). The
matrices $\mK$ and $\mW^{-1}(\vf_\newtonidx)$ are evaluated lazily. We thus report the
runtime costs when the matrix-vector products are actually computed. For an in-depth
discussion of the computational costs, see \cref{sec:details_cost_analysis}.
 \end{algorithm*}

\textbf{Inner Loop: Computation-Aware GP Regression via \itergp.}
Reframing the Newton iteration as sequential GP regression does not yet solve the need
for linear solves with a matrix of size \(\numtraindata\numclasses \times
\numtraindata\numclasses\). However, it allows us to leverage recent advances for GP
regression, specifically the \itergp algorithm introduced by
\citet{Wenger2022PosteriorComputational}. \itergp is matrix-free, \ie only relies on
matrix-vector products \(\vs \mapsto \KWinvf{\vf_\newtonidx} \, \vs\), reducing the
required memory from quadratic to linear, and efficiently exploits modern parallel GPU
hardware \citep{Charlier2021}.

Internally, \itergp uses a probabilistic linear solver (PLS)
\citep{Hennig2015,Cockayne2019,Wenger2020} to iteratively compute a Gaussian belief over
the so-called \textit{representer weights}, \ie over the solution of the linear system
$\KWinvf{\vf_\newtonidx} \, \vv = \pseudotargetsf{\vf_\newtonidx} - \vm$. In each solver
iteration, indexed by $\solveridx$, the belief
$\gaussianpdf{\vv}{\vv_\solveridx}{\mOmega_\solveridx}$ is updated by conditioning the
Gaussian on a one-dimensional projection $\alpha_\solveridx = \action_\solveridx^\top
\residual_{\solveridx - 1}$ of the preceding residual $\residual_{\solveridx - 1} =
\pseudotargetsf{\vf_\newtonidx} - \vm - \KWinvf{\vf_\newtonidx} \vv_{\solveridx - 1}$.
The vector $\action_\solveridx \gets \textsc{Policy}()$ is called an \textit{action} and
is generated by a user-specified \textit{policy}. The policy determines the behavior of
the solver by ``weighting'' the residual $\residual_{\solveridx - 1}$ for specific
\datapoints (we discuss the role of the policy in the context of NCGPs in
\cref{sec:policy}.). \citet[Tab.~1]{Wenger2022PosteriorComputational} lists policies
that have classic counterparts, \eg unit vectors $\action_\solveridx \gets
\ve_\solveridx$ correspond to partial Cholesky and residual actions $\action_\solveridx
\gets \residual_{\solveridx - 1}$ to conjugate gradients \citep{Hestenes1952}.

The belief over the representer weights translates into an approximate GP posterior over
the latent function,
\begin{align}
  m_{\newtonidx,\solveridx}(\cdot)
  &\defeq m(\cdot) + \kernelfn(\cdot, \mX) \vv_\solveridx
  \label{eq:inference_PLS_iteration_mean}
  \\
    \kernelfn_{\newtonidx,\solveridx}(\cdot, \cdot)
    &\defeq \kernelfn(\cdot, \cdot)
    - \kernelfn(\cdot, \mX) \, \mC_\solveridx \kernelfn(\mX, \cdot),
  \label{eq:inference_PLS_iteration_covar}
\end{align}
where \(\vv_\solveridx = \mC_\solveridx (\pseudotargetsf{\vf_\newtonidx} - \vm)\).
Crucially, by \citet[Thm.~2]{Wenger2022PosteriorComputational}, the posterior covariance
in \cref{eq:inference_PLS_iteration_covar} \emph{exactly} quantifies the error in each
approximate Newton step introduced by only using \emph{limited computational resources},
\ie running the linear solver for \(\solveridx \ll \numtraindata\numclasses\)
iterations. This reduces the time complexity to \(\bigO{\solveridx
\numtraindata^2\numclasses^2}\). The approximate precision matrix in
\cref{eq:inference_PLS_iteration_mean,eq:inference_PLS_iteration_covar},
\begin{equation}
  \mC_\solveridx
  =
  \mS_\solveridx
  (\mS_\solveridx^\top \KWinvf{\vf_\newtonidx} \mS_\solveridx)^{-1}
  \mS_\solveridx^\top
\label{eq:C_j}
\end{equation}
with \(\mS_\solveridx = \begin{pmatrix} \vs_1, \dots, \vs_\solveridx \end{pmatrix} \in
\R^{\numtraindata\numclasses \times \solveridx}\) has rank \(\solveridx\) and approaches
$\KWinvf{\vf_\newtonidx}^{-1}$ as \(\solveridx \to \numtraindata\numclasses\).
Intuitively, it projects $\KWinvf{\vf_\newtonidx}$ onto the subspace spanned by the
actions $\mS_\solveridx$, then inverts and projects the result back into the original
space.
The \itergp algorithm is given in \cref{alg:inner}.

\textbf{The Marginal Uncertainty Decreases in the Inner Loop.}
As we perform more solver iterations, the marginal uncertainty captured by the posterior
covariance (\cref{eq:inference_PLS_iteration_covar}) contracts, \ie for each $i$ it
holds (element-wise) that
$
  \diag(\kernelfn_{\newtonidx, \solveridx}(\vx, \vx))
  \geq
  \diag(\kernelfn_{\newtonidx, k}(\vx, \vx))
$
for any $k \geq \solveridx$ and arbitrary $\vx$. This is because the approximate
precision matrix $\mC_\solveridx$ grows in rank with each solver iteration. For a
detailed derivation, see \cref{subsec:uncertainty_decreases_inner_loop}.

\begin{algorithm*}[t!]  \caption{\textbf{\iterncgp Inner Loop: \itergp with a Virtual Solver Run.}\label{alg:inner}}
  \small
  \algrenewcommand{\algorithmiccomment}[3]
{\hfill
    {
        \small \textcolor{gray}{#1 \hspace{0.3em}}
        \makebox[9.5em][l]{#2}   \makebox[5em][l]{#3}  }
}

\textbf{Input:} GP prior $\gp{m}{\kernel}$, training data $(\mX, \vy)$, $\vm$, access to
products with $\mK$ and $\mW^{-1}$, pseudo targets $\pseudotargets$, buffers
$\textcolor{Recycling}{\mS}, \textcolor{Recycling}{\mT}$\\
\textbf{Output:} GP posterior
$\gp{m_{\newtonidx, \solveridx}}{\kernelfn_{\newtonidx, \solveridx}}$
\begin{algorithmic}[1]
    \Procedure{\itergp}{$m, \kernelfn, \mX, \vy, \vm, \mK, \mW^{-1},
    \pseudotargets, \textcolor{Recycling}{\mS},
    \textcolor{Recycling}{\mT}$}
    \Comment{}{\textbf{Time}}{\textbf{Memory}}
    \State \textcolor{Recycling}{
        $\mC_0, \mS, \mT \gets$ \textsc{VirtualSolverRun}($\mS, \mT, \mW^{-1}$)
    }
    \Comment{See \cref{alg:virtual_solver_run}}{}{}
    \State $\vv_0 \gets \mC_0 (\pseudotargets - \vm)$
    \Comment{$\vv_0:$ Consistent initial iterate}
    {$\bigO{\bufferlimit\numtraindata\numclasses}$}
    {$\bigO{\numtraindata\numclasses}$}
    \For{$\solveridx = 1, 2, 3, \dots$ \textbf{while not} \textsc{InnerStoppingCriterion}()}
        \State $\residual_{\solveridx - 1} \gets (\pseudotargets{}\!-\!\vm)\!-\!\mK \vv_{\solveridx-1}\!-\!\mW^{-1} \vv_{\solveridx-1}$
        \Comment{$\residual_{\solveridx - 1}:$ Residual vector}
        {$\bigO{\timeK\!+\!\timeWinv\!+\!\numtraindata\numclasses}$}
        {$\bigO{\numtraindata\numclasses}$}
        \State $\action_\solveridx \gets \textsc{Policy}()$
        \Comment{Select action $\action_\solveridx$ via policy}
        {$\bigO{\timepolicy}$}
        {$\bigO{\numtraindata\numclasses}$}
        \State \textcolor{Recycling}{Append $\action_\solveridx$ to buffer
        $\mS \gets (\mS, \action_\solveridx) \in \R^{\numtraindata\numclasses \times \buffersize}$
        }
        \Comment{}
        {}
        {$\bigO{\buffersize \numtraindata\numclasses}$}
        \State $\observ_\solveridx \gets \action_\solveridx^\top \residual_{\solveridx-1}$
        \Comment{$\observ_\solveridx:$ Observation is projection of residual onto action}
        {$\bigO{\numtraindata\numclasses}$}
        {$\bigO{1}$}
        \State $\vt_\solveridx \gets \mK \action_\solveridx$
        \Comment{First term in $\KWinv \action_\solveridx\!=\!  \mK\action_\solveridx\!+\!\textcolor{lightgray}{\mW^{-1}\action_\solveridx}$}
        {$\bigO{\timeK}$}
        {$\bigO{\numtraindata\numclasses}$}
        \State \textcolor{Recycling}{Append $\vt_\solveridx$ to buffer
        $\mT \gets (\mT, \vt_\solveridx) \in \R^{\numtraindata\numclasses \times \buffersize}$
        }
        \Comment{}
        {}
        {$\bigO{\buffersize \numtraindata\numclasses}$}
        \State $\vz_\solveridx \gets \vt_\solveridx\!+\!\mW^{-1} \action_\solveridx$
        \Comment{Second term in $\KWinv \action_\solveridx\!=\!  \textcolor{lightgray}{\mK\action_\solveridx}\!+\!\mW^{-1}\action_\solveridx$}
        {$\bigO{\timeWinv\!+\!\numtraindata\numclasses}$}
        {$\bigO{\numtraindata\numclasses}$}
        \State
        $\searchdir_\solveridx \gets
        \action_\solveridx - \kernmatinvapprox_{\solveridx-1}\vz_\solveridx$
        \Comment{$\searchdir_\solveridx:$ Search direction}
        {$\bigO{\buffersize\numtraindata\numclasses}$}
        {$\bigO{\numtraindata\numclasses}$}
        \State
        $\searchdirsqnorm_\solveridx \gets \vz_\solveridx^\top \searchdir_\solveridx$
        \Comment{$\searchdirsqnorm_\solveridx:$ Normalization constant}
        {$\bigO{\numtraindata\numclasses}$}
        {$\bigO{1}$}
        \State $\matrixroot_\solveridx \gets (\matrixroot_{\solveridx - 1}, \nicefrac{1}{\sqrt{\searchdirsqnorm_\solveridx}}\, \searchdir_\solveridx) \in \R^{\numtraindata\numclasses \times \buffersize}$
        \Comment{Append column}
        {$\bigO{\numtraindata\numclasses}$}
        {$\bigO{\buffersize\numtraindata\numclasses}$}
        \State $\mC_\solveridx \gets \matrixroot_\solveridx \matrixroot_\solveridx^\top$
        \Comment{Rank $\buffersize$ approximation $\mC_\solveridx \approx\!\KWinv^{-1}$}
        {}
        {}
        \State $\rweightsapprox_\solveridx \gets \rweightsapprox_{\solveridx-1} + \frac{\alpha_\solveridx}{\searchdirsqnorm_\solveridx}\searchdir_\solveridx  $
        \Comment{$\rweightsapprox_\solveridx$ Updated representer weights estimate}
        {$\bigO{\numtraindata\numclasses}$}
        {$\bigO{\numtraindata\numclasses}$}
        \State $m_{\newtonidx,\solveridx}(\cdot)
        \gets m(\cdot) + \kernel(\cdot, \traindata) \rweightsapprox_\solveridx$
        \Comment{\Cref{eq:inference_PLS_iteration_mean}}
        {$\bigO{\numtraindata\numtestdata\numclasses^2}$}
        {$\bigO{\numtestdata\numclasses}$}
        \State $\kernel_{\newtonidx,\solveridx}(\cdot, \cdot)
        \gets \kernel(\cdot, \cdot)	- \kernel(\cdot, \traindata)
        \kernmatinvapprox_\solveridx \kernel(\traindata, \cdot)$
        \Comment{\Cref{eq:inference_PLS_iteration_covar}}
        {$\bigO{\buffersize(\numtraindata\!+\!\numtestdata)\numtestdata\numclasses^2}$}
        {$\bigO{\numtestdata^2\numclasses^2}$}
    \EndFor
    \State \Return
    $\mathcal{GP}(m_{\newtonidx, \solveridx}, \kernel_{\newtonidx, \solveridx})$ and $\vv_\solveridx$
    \EndProcedure
\end{algorithmic}

Instructions in \textcolor{Recycling}{blue} are needed for recycling (see
\cref{sec:recycling}). $\mC_\solveridx$ is represented via its root
$\matrixroot_\solveridx$ and evaluated lazily. We thus report the runtime costs when the
matrix-vector products are actually computed. The costs for evaluating the posterior GP
$\gp{m_{\newtonidx,\solveridx}}{\kernel_{\newtonidx,\solveridx}}$ are based on
$\numtestdata$ test \datapoints $\mX_\symboltestdata \in \R^{\numtestdata \times D}$.
For an in-depth discussion of the computational costs, see \cref{sec:details_cost_analysis}.
 \end{algorithm*}

\textbf{Summary.} Finding the posterior mode $\vf_\text{MAP}$ and the corresponding
predictive distributions (\cref{eq:predictive_mean,eq:predictive_variance}) can be
viewed from different angles. Through an optimization lens, we use Newton updates, each
maximizing a local quadratic approximation of the log-posterior. From a probabilistic
perspective, we solve a sequence of \textit{related} GP regression problems and \itergp
enables us to propagate a probabilistic estimate of the latent function throughout the
\textit{entire} optimization process.

For Gaussian likelihoods, the LA (\cref{eq:laplace_approximation}) is exact and a single
Newton step suffices. Consequently, our framework generalizes \itergp to arbitrary
log-concave likelihoods (\cref{theo:extension_of_itergp}). We now explore the role of
the policy and its potential for actions tailored to specific problems
(\cref{sec:policy}). We also leverage the relatedness of GP regression problems in the
outer loop for further speedups (\cref{sec:recycling}) and introduce a mechanism to
control \iterncgp{}'s memory usage (\cref{sec:compression}).

\subsection{Policy Choice: Targeted Computations}
\label{sec:policy}

\cref{alg:inner} defines a \textit{family} of inference algorithms. Its instances,
defined by a concrete action policy, generally behave quite differently. To better
understand what effect the sequence of actions \(\mS_\solveridx = \begin{pmatrix}
\vs_1, \dots, \vs_\solveridx \end{pmatrix} \in \R^{\numtraindata\numclasses \times
\solveridx}\) has on \iterncgp, we consider the following examples.

\textbf{Unit Vector Policy $ = $ Subset of Data (SoD).}
Choosing the actions \(\vs_\solveridx = \ve_\solveridx\) to be unit vectors with all
zeros except for a one at entry \(\solveridx\), corresponds to (sequentially)
conditioning on the first \(\solveridx\) \datapoints in the training data in each GP
regression subproblem, since \(\kernelfn(\cdot, \mX)\mS_\solveridx = \kernelfn(\cdot,
\mX_{1:\solveridx})\). Therefore this policy is equivalent to simply using a subset
\(\mX_{1:\solveridx} \in \R^{\solveridx \times \inputdim}\) of the data and performing
\textit{exact} GP regression (\eg via a Cholesky decomposition) in each Newton iteration
in \cref{alg:outer}. This basic policy shows how actions \emph{target computation} as
illustrated in the top row of \cref{fig:policy_choice_mean}.

\textbf{(Conjugate) Gradient Policy.}
Instead of targeting individual \datapoints, we can also specify \textit{weighted}
linear combinations of them to target the data more globally.
E.g.~, using the current residual \(\vs_\solveridx = \vr_{\solveridx-1} =
\pseudotargetsf{\vf_\newtonidx} - \vm - \KWinvf{\vf_\newtonidx}\vv_{\solveridx-1}\),
approximately targets those data most, where the posterior mean prediction is far
off.\footnote{Since \(\vr_{\solveridx-1} \approx \pseudotargetsf{\vf_\newtonidx} -\vm -
\mK \vv_{\solveridx-1} = \pseudotargetsf{\vf_\newtonidx} - m_{\newtonidx,
\solveridx-1}(\mX)\).} As \citet{Wenger2022PosteriorComputational} show, this
corresponds to using conjugate gradients (CG) \citep{Hestenes1952} to estimate the
posterior mean. This policy is illustrated in the bottom row of
\cref{fig:policy_choice_mean}.

\begin{figure}[tb!]
  \includegraphics[width=0.99\linewidth]{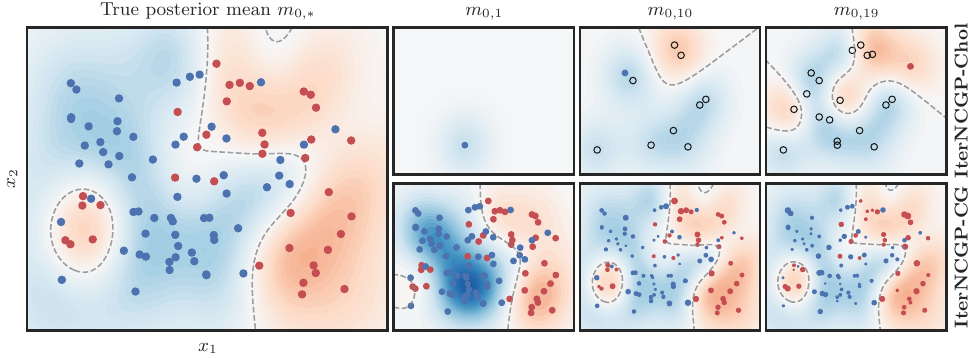}
  \caption{\textbf{Different  \iterncgp Policies Applied to GP Classification.}
  \emph{(Left)} The true posterior mean $m_{0,*}$
  (\,\colorgradientbox[MPLblue_shaded]{MPLred_shaded}\,) for binary classification
  (\colordot{PlotBlue}/\colordot{PlotRed}) and its decision boundary
  (\dashedcolorline{gray}). \emph{(Right)} Current posterior mean estimate after $1,
  10,$ and $19$ iterations with the unit vector policy (\emph{Top}) and the CG policy
  (\emph{Bottom}). Shown are the \datapoints selected by the policy in this iteration
  with the dot size indicating their relative weight. For \iterncgp-Chol, \datapoints
  are targeted one by one and previously used \datapoints are marked with
  (\colorcircle{black}). Details in \cref{sec:details_binary_classification}}
  \label{fig:policy_choice_mean}
\end{figure}

\subsection{Recycling: Reusing Computations}
\label{sec:recycling}

Using \itergp with a suitable policy for GP inference allows us to solve each GP
regression problem more efficiently. However, for NCGP inference, we must solve
\emph{multiple} regression problems---one per mode-finding step.
\Cref{fig:pseudo_targets_noise} suggests that GP posteriors across steps are highly
similar. Leveraging this observation, we develop a novel approach, designed specifically
for the NCGP setting, that \emph{efficiently recycles costly computations} between outer
loop steps (pseudo code in \cref{alg:virtual_solver_run}).

\begin{algorithm*}[tb!]
  \caption{\textbf{Recycling: Virtual Solver Run with Optional
  Compression.}\label{alg:virtual_solver_run}}
  \small
  \algrenewcommand{\algorithmiccomment}[3]
{\hfill
    {
        \small \textcolor{gray}{#1 \hspace{0.3em}}
        \makebox[9em][l]{#2}   \makebox[4.5em][l]{#3}  }
}

\textbf{Input:} Buffers $\mS, \mT \in \R^{\numtraindata\numclasses \times \buffersize}$, access to products with $\mW^{-1}$,
compression parameter $\bufferlimit \leq \buffersize$ (optional)\\
\textbf{Output:} $\mC_0$, updated buffers $\mS, \mT$
\begin{algorithmic}[1]
    \Procedure{\textsc{VirtualSolverRun}}{$\mS, \mT, \mW^{-1}$}
    \Comment{}{\textbf{Time}}{\textbf{Memory}}
    \State $\mM \gets \mS^\top (\mT + \mW^{-1} \mS)$
    \Comment{$\mM\!=\!\mS^\top\!(\mK + \mW^{-1}) \mS \in \R^{\buffersize \times \buffersize}$}
    {$\bigO{\buffersize\timeWinv\!+\!\buffersize^2\!\numtraindata\numclasses}$}
    {$\bigO{\buffersize^2}$}
    \State $\mU, \mLambda \gets$ \textsc{ED}($\mM$),
    \Comment{Eigendecomposition $\mM = \mU \mLambda \mU^\top$}
    {$\bigO{\buffersize^3}$}
    {$\bigO{\buffersize^2}$}
    \Statex \hskip\algorithmicindent
    $
    \mU\!=\!(\vu_1,\!\dots,\!\vu_\buffersize),
    \mLambda\!=\!\diag(\lambda_1,\!\dots,\!\lambda_\buffersize)\!
    \in\!\R^{\buffersize\!\times\!\buffersize},
    \lambda_1\!\geq\!\dots\!\geq\!\lambda_\buffersize$
    \vspace{1ex}
    \Procedure{\textsc{Compression}}{$\mU, \mLambda, \bufferlimit$}
    \Comment{}{}{}
    \State $\mU \gets (\vu_1, \dots, \vu_\bufferlimit)$,
    $\mLambda \gets \diag(\lambda_1, \dots, \lambda_\bufferlimit)$
    \Comment{Truncation to $\bufferlimit$ eigenpairs}
    {}
    {$\bigO{\buffersize\bufferlimit}$}
    \EndProcedure
    \State $\mS \gets \mS \mU$, $\mT \gets \mT \mU$
    \Comment{Update buffers}
    {$\bigO{\buffersize\bufferlimit\numtraindata\numclasses}$}
    {$\bigO{\bufferlimit\numtraindata\numclasses}$}
    \State $\matrixroot_0 \gets \mS \mLambda^{-1/2}$
    \Comment{Construct root $\mC_0 = \matrixroot_0 \matrixroot_0^\top = \mS \mLambda^{-1} \mS^\top$}
    {$\bigO{\bufferlimit^2\numtraindata\numclasses}$}
    {$\bigO{\bufferlimit\numtraindata\numclasses}$}
    \State \Return $\mC_0 \gets \matrixroot_0 \matrixroot_0^\top$ and $\mS, \mT$
    \Comment{$\mC_0$ has rank $\bufferlimit$}
    {}
    {}
    \EndProcedure
\end{algorithmic}
$\mC_0$ is \textit{never} formed explicitly in memory but evaluated lazily via its root $\matrixroot_0$, \ie $\vw \mapsto \mC_0 \vw = \matrixroot_0 (\matrixroot_0^\top \vw)$.
 \end{algorithm*}

The cost of \iterncgp is dominated by repeated matrix-vector products with $\mK$ (see
\cref{sec:cost_analysis}). However, these costly operations can be recycled and used
over multiple Newton steps: Consider the matrix-vector products with an action vector
$\action$ in the first and second mode-finding step as an example:
\begin{align*}
  \qquad\qquad\qquad\qquad\qquad
  &\text{Step $\newtonidx = 0$:}
  &
  \action
  &\mapsto
  \KWinvf{\vf_0} \action
  = \mK \action + \mW(\vf_0)^{-1} \action \\[-1ex]
  &\text{Step $\newtonidx = 1$:}
  &
  \action
  &\mapsto
  \KWinvf{\vf_1} \action
  = \mK \action + \mW(\vf_1)^{-1} \action.
  \qquad\qquad\qquad\qquad\qquad
\end{align*}
Since $\mK$ is independent of $\vf_\newtonidx$, the product $\mK \action$ is
\textit{shared} among both operations.

\textbf{Virtual Solver Run.} Assume we have used $\buffersize$ action vectors
$(\action_1, \dots, \action_\buffersize) \eqdef \mS \in \R^{\numtraindata\numclasses
\times \buffersize}$ in step $\newtonidx = 0$, and buffered the matrix-vector products
$(\mK\action_1, \dots, \mK\action_\buffersize) = \mK \mS \eqdef \mT$. In the next Newton
step $\newtonidx = 1$ we apply the \textit{same} actions to a \textit{new} linear system
of equations. From \cref{eq:C_j} we obtain
\begin{equation}
  \mC
  =
  \mS \mM^{-1} \mS^\top
  \;\; \text{with} \;\;
  \mM
  \defeq
  \mS^\top (\mK\mS + \mW (\vf_1)^{-1}\mS)
  =
  \mS^\top (\mT + \mW (\vf_1)^{-1}\mS).
  \label{eq:recycling}
\end{equation}
So, we can \textit{imitate} a solver run with the previous actions $\mS$ and construct
$\mC$ without ever having to multiply with $\mK$. The associated computational costs
comprise memory for the two buffers $\mS$, $\mT$ as well as the runtime costs for
matrix-matrix products in \cref{eq:recycling} and inverting $\mM$. This virtual solver
run is generally orders of magnitude cheaper than running the solver from scratch with
new actions (details in \cref{sec:details_cost_analysis}).
Within \itergp (\cref{alg:inner}), we can use this matrix as an initial estimate
$\mC_0 \gets \mC$ of the precision matrix. Subsequently, the algorithm can proceed as
usual with new actions.

The presented recycling approach can easily be extended to all Newton steps. Whenever
$\mK$ is multiplied with an action vector, the vector itself and the resulting vector
are appended to the respective buffers $\mS$ and $\mT$. For each Newton step, an initial
$\mC_0$ can be constructed via \cref{alg:virtual_solver_run}.

\textbf{Numerical Perspective.}
Crucially, the above strategy does not affect the solver's convergence properties: From
a numerical linear algebra viewpoint, the strategy above is a form of \emph{subspace
recycling} \citep{Parks2006RecyclingKrylov}. Specifically, $\mC_0$, as described above,
defines a \textit{deflation preconditioner} \citep{Frank2001ConstructionDeflationbased}:
The projection of the initial residual $\vr_0 = (\pseudotargets - \vm) - \KWinv \vv_0$
for the first iterate $\vv_0 = \mC_0 (\pseudotargets - \vm)$ onto the subspace
$\vspan\{\mS\}$ spanned by the actions is zero (see
\cref{sec:details_virtual_solver_run} for details). That means, the solution within the
subspace $\vspan\{\mS\}$ is already perfectly identified at initialization.

\textbf{Probabilistic Perspective.}
Via \cref{eq:inference_PLS_iteration_covar}, we can quantify the effect of $\mC_0$ on
the total marginal uncertainty of predictions at the training data
$\Tr(\kernelfn_{\newtonidx,0}(\mX, \mX)) = \Tr(\mK) - \Tr(\mK \mC_0 \mK)$. Assuming
observation noise $\mW^{-1} = \mzero$ and all actions in $\mS$ eigenvectors of $\KWinv =
\mK$, it simplifies to
\begin{equation}
  \Tr(\kernelfn_{\newtonidx,0}(\mX, \mX))
  =
  \Tr(\mK) - \Tr(\mM),
  \label{eq:total_marginal_variance_special_case}
\end{equation}
see \cref{sec:details_virtual_solver_run}.
The second term $\Tr(\mM)$ describes the \textit{reduction} of the prior uncertainty due
to $\mC_0$. It can be maximized (which is our goal) when $\mS$ contains those
eigenvectors of $\KWinv$ with the largest eigenvalues. We take this insight as
motivation for a buffer compression approach that we describe next.

\subsection{Compression: Memory-Efficient Beliefs}
\label{sec:compression}

Whenever $\mK$ is applied to an action vector, the buffers $\mS, \mT \in
\R^{\numtraindata\numclasses \times \buffersize}$ grow by \(\numtraindata\numclasses\)
entries. To limit memory requirements for large-scale data, we propose a compression
strategy (see \cref{alg:virtual_solver_run}).

\textbf{Compression via Truncation.} In \cref{alg:virtual_solver_run}, $\mM^{-1} \in
\R^{\buffersize \times \buffersize}$ is computed via an eigendecomposition $\mM = \mU
\mLambda \mU$, such that $\mC_0 = \matrixroot_0 \matrixroot_0^\top$ can be represented
via its matrix root $\matrixroot_0 \defeq \mS\mU\mLambda^{-1/2}$ for efficient storage
and matrix-vector multiplies. To limit memory usage, we can use a \textit{truncated}
eigendecomposition of $\mM$. Based on the intuition we gained from
\cref{eq:total_marginal_variance_special_case}, it makes sense to keep the
\textit{largest} eigenvalues (to maximize the trace) and corresponding eigenvectors.
Keeping the $\bufferlimit$ largest eigenvalues/-vectors yields a rank $\bufferlimit$
approximation $\tilde{\mM} = \tilde{\mU}\tilde{\mLambda}\tilde{\mU}$ of $\mM$.

\textbf{Compression as Re-Weighting Actions.} Forming $\mC_0 = \mS \tilde{\mM}^{-1}
\mS^\top$ from the above approximation is equivalent to a virtual solver run with the
modified buffers $\tilde{\mS} = \mS \mU \in \R^{\numtraindata\numclasses \times
\bufferlimit}$, $\tilde{\mT} = \mK (\mS \mU) = \mT \mU \in \R^{\numtraindata\numclasses
\times \bufferlimit}$ in \cref{eq:recycling}. This shows that the truncated
eigendecomposition effectively re-weights the previous $\buffersize$ actions to form
$\bufferlimit$ new ones---and the weights are the eigenvectors from $\mM$ that maximize
the uncertainty reduction. The limit $\bufferlimit$ on the buffer size controls the
\emph{memory usage} as well as the rank of $\mC_0$ and thereby the
\textit{expressiveness} of the associated belief (see \cref{fig:virtual_solver_run}).

\begin{figure}[tb!]
  \begin{minipage}[t]{0.5\textwidth}
    \vspace{0ex}
    \centering
    \includegraphics[width=\linewidth]{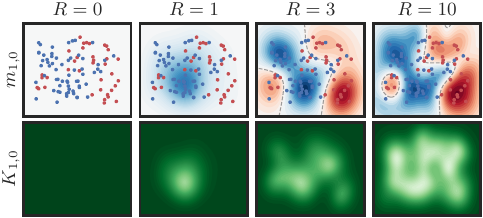}
  \end{minipage}
  \hfill
  \begin{minipage}[t]{0.47\textwidth}
    \vspace{0ex}
    \caption{\textbf{Compressed Beliefs.}
    Recycled initial beliefs in the \textit{second} Newton step ($\newtonidx  =  1$)
    with means $m_{1, 0}$ \emph{(Top)} and (co-)variance functions $\kernelfn_{1, 0}$
    \emph{(Bottom)} using compression with different buffer sizes $\bufferlimit \in \{0,
    1, 3, 10\}$. Buffer size \(\bufferlimit=0\) is equivalent to not using recycling.
    The larger the buffer size/rank of $\mC_0$, the more expressive the belief. Details
    in \cref{sec:details_binary_classification}.}
  \label{fig:virtual_solver_run}
  \end{minipage}
\end{figure}

\subsection{Cost Analysis of \iterncgp}
\label{sec:cost_analysis}

\iterncgp's total runtime is dominated by the repeated application of $\mK$ in
\cref{alg:inner}, \ie $\bigO{\totalsolveriters \timeK}$, with $\totalsolveriters$
describing the \textit{total} number of solver iterations over \textit{all} Newton
steps. $\timeK$ denotes the cost of a single matrix-vector product with $\mK$.
Typically, $\timeK$ is quadratic in the number of training \datapoints.
In terms of memory, the buffers $\mS$, $\mT$ and the matrix root $\matrixroot$ are the
decisive factors with $\bigO{\buffersize\numtraindata\numclasses}$. Without compression,
their final size is $\buffersize = \totalsolveriters$. Otherwise, their maximum size
is given by the sum of the rank bound $\bufferlimit$ and the maximum solver iterations
in \cref{alg:inner} (\cref{sec:details_cost_analysis} provides an in-depth discussion of
runtime and memory costs).

\section{Related Work}
\label{sec:related_works}

The Laplace approximation \citep{Spiegelhalter1990SequentialUpdating,
MacKay1992EvidenceFramework, Bishop2006, Rue2009ApproximateBayesian} is commonly used
for approximate inference in (Bayesian) Generalized Linear Models. Here, we consider the
function-space generalization of Bayesian Generalized Linear Models, namely
non-conjugate GPs, for which a multitude of approximate methods have been proposed,
arguably the most popular being variational approaches
\citep[\eg][]{Khan2012FastBayesian}, such as \svgp \citep{Titsias2009,
Hensman2015ScalableVariational}. In contrast, to address the computational shortcomings
of NCGPs on large \datasets, we leverage iterative methods to obtain and efficiently
update low-rank approximations. Similar approaches were used previously to accelerate
the conjugate Gaussian special case \citep{Cunningham2008, Murray2009,
Guahniyogi2015BayesianCompressed, Gardner2018GPyTorchBlackbox, Wang2019,
Wenger2022PreconditioningScalable}, binary classification
\citep{Zhang2014RandomProjections} and general Bayesian linear inverse problems
\citep{Spantini2015OptimalLowrank}. \citet{Trippe2019LRGLMHighDimensional} is closest in
spirit to our approach if viewed from a weight-space perspective. Their choice of
low-rank projection corresponds to a specific policy in our framework. Our approach not
only enables the use of policies that are more suited to the given link function, but
also saves additional computation, as well as memory, via recycling and compression. In
each Newton iteration, the posterior for the current regression problem is approximated
via \itergp \citep{Wenger2022PosteriorComputational}, which internally uses a
probabilistic linear solver \citep{Hennig2015,Cockayne2019,Wenger2020}. Therefore,
\iterncgp is a probabilistic numerical method \citep{Hennig2015a, Cockayne2019a,
Oates2019, Hennig2022}: It quantifies uncertainty arising from limited computation.

\section{Experiments}
\label{sec:experiments}

We apply \iterncgp to a Poisson regression problem to explore the trade-off between the
number of (outer loop) mode-finding steps and (inner loop) solver iterations
(\cref{sec:exp_poisson_regression}). In \Cref{sec:exp_large_scale_gpc}, we demonstrate
our algorithm's scalability and the impact of compression on performance.

\subsection{Poisson Regression}
\label{sec:exp_poisson_regression}

Consider count data $\vy \in \N_0^\numtraindata$ generated from a Poisson likelihood
with unknown rate $\lambda \colon \inputspace \to \R_+$. The log rate $f$ is modeled by
a GP. See \cref{sec:details_poisson_regression} for details.

\textbf{Data \& Model.} We generate a synthetic \dataset by (i) sampling the log rate
from a GP with an RBF kernel (ii) transforming it into the latent rate $\lambda$ by
exponentiation, and (iii) sampling counts $y_n \in \N_0$ from the Poisson distribution
with rate $\lambda(\vx_n)$. The functions $f$, $\lambda$, and the resulting count data
are shown in \cref{fig:poisson_regression} \emph{(Right)}. Our model uses the same RBF
prior GP to avoid model mismatch.

\textbf{Newton Steps vs.~Solver Iterations.}
From a practical standpoint, the performance achievable within a given budget of solver
iterations is highly relevant: How many linear solver iterations should be performed for
each regression problem before updating the problem to maximize performance? To
investigate this, we use \iterncgp-CG and distribute $100$ iterations uniformly over
$\{5, 10, 20, 100\}$ outer loop steps. Each run uses recycling without compression and
is repeated $10$ times.

\begin{figure}[tb!]
  \includegraphics[width=0.99\linewidth]{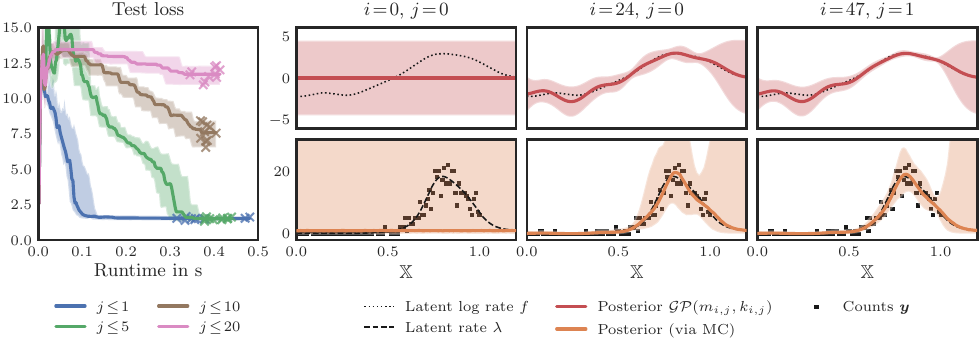}
  \vspace{-0.5em}
  \caption{\textbf{Poisson Regression with \iterncgp.}
    \emph{(Left)} Test loss performance for \iterncgp-CG with recycling and four
    schedules ($\solveridx  \leq  1, 5, 10$ or $20$) over $100, 20, 10$ or $5$ steps
    (always using the same total budget of $100$ iterations). For each schedule, the
    median (solid line) and min/max (shaded area) over $10$ runs are reported. The
    crosses indicate the end of each run.
    \emph{(Right)} Posterior
    $\gp{m_{\newtonidx,\solveridx}}{k_{\newtonidx,\solveridx}}$ for the latent log
    rate $f$ \emph{(Top)} and the corresponding belief about the rate $\lambda$
    \emph{(Bottom)} computed via MC at three timepoints during a run of \iterncgp.
    The shaded 95\% credible intervals show how stopping early trades less
    computation for increased uncertainty. Details in
    \cref{sec:details_poisson_regression}.}
  \label{fig:poisson_regression}
  \vspace{-1em}
\end{figure}

\textbf{Results.} \cref{fig:poisson_regression} \emph{(Left)} indicates that the
strategy with a single iteration per step is the most efficient. An explanation might be
that there is no reason to spend compute on an ``outdated'' regression problem that
could be updated instead. Of course, this only applies if recycling is used, such that
the \textit{effective} number of actions accumulates.
As long as $\buffersize \ll \numtraindata$, the cost due to repeated recycling
($\bigO{\numtraindata}$) is dwarfed by the cost of products with $\mK$
($\bigO{\numtraindata^2}$).
\cref{fig:poisson_regression} \emph{(Right)} shows an \iterncgp-CG run with one
iteration per step. As we spend more computational resources, our estimates approach the
underlying latent function and, where data is available, the uncertainty contracts.

\subsection{Large-Scale GP Multi-Class Classification}
\label{sec:exp_large_scale_gpc}

Here, we showcase \iterncgp's scalability. See \cref{sec:details_large_scale_gpc} for
details.

\textbf{Data \& Model.} We generate $\numtraindata =  10^5$ \datapoints from a Gaussian
mixture model with $\numclasses = 10$ classes.
We use the softmax likelihood and assume independent GPs (each equipped with a
\matern{3} kernel) for the $\numclasses$ outputs of the latent function.
While this experiment uses synthetic data, the latent function is \textit{not} drawn
from the assumed GP model and thus the kernel is \textit{not} perfectly identified.
Also note that, if we formed $\smash{\KWinv}$ in (working) memory explicitly, this would
require $(\numtraindata\numclasses)^2 \cdot\,8~\text{byte} = 8000~\text{GB}$ (in double
precision). Solving the linear systems \textit{precisely}, \eg via Cholesky
decomposition, is therefore infeasible, whereas our family of methods is matrix-free and
can still be applied.

\textbf{Methods.}
We compare the subset of data (SoD) approach from \cref{sec:policy}, the popular
variational approximation \svgp \citep{Titsias2009,Hensman2015ScalableVariational} and
our \iterncgp-CG. For SoD, we materialize $\smash{\KWinv}$ in memory and compute its
Cholesky decomposition. Four different subset sizes are used---the largest one
$\numtraindata_\text{sub} = 2000$ requires $3.2~\text{GB}$ of memory for
$\smash{\KWinv}$. For \svgp, we use the implementation provided by \gpytorch{}
\citep{Gardner2018GPyTorchBlackbox}. We optimize the ELBO for $10^4~\text{seconds}$
using \adam with batch size $1024$ and determine suitable hyperparameters via grid
search over the learning rate $\alpha$ and the number of inducing points $U$. Only the
best three settings are included in our final benchmark. \iterncgp-CG is applied to the
\textit{full} training set with recycling and $\bufferlimit \in \{\infty, 10\}$. The
number of solver iterations is limited by $\solveridx \leq 5$. We use \keops
\citep{Charlier2021} and \gpytorch for fast kernel-matrix multiplies. For this work, we
consider kernel hyperparameter optimization out of scope---all methods therefore use the
same fixed hyperparameters. The benchmark is run on an \textsc{NVIDIA A100 GPU}.

\begin{figure}[tb!]
  \includegraphics[width=0.99\linewidth]{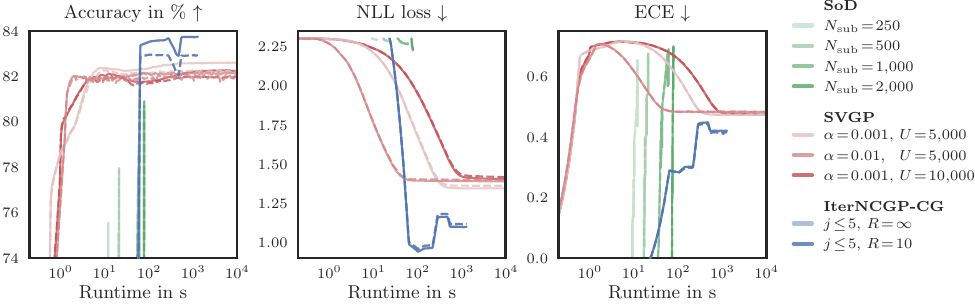}
\caption{\textbf{Large-Scale GP Classification.}
    Comparison of SoD,
    \svgp with learning rate $\alpha \in \{0.001, 0.01, 0.05\}$ and $U \in \{1000,
    2500, 5000, 10000\}$ inducing points (showing only the best three runs) and
    \iterncgp-CG with $\bufferlimit \in \{\infty, 10\}$ on a classification problem
    with \(\numtraindata = 10^5\) training points and \(\numclasses = 10\) classes
    in terms of accuracy \emph{(Left)}, NLL loss \emph{(Center)} and ECE
    \emph{(Right)}. Performance metrics are averaged over five runs and are shown as
    solid (training set) or dashed (test set) lines.
    \iterncgp-CG performs best in \textit{all} three performance metrics, with
    minimal memory requirements. The two variants with $\bufferlimit \in \{\infty,
    10\}$ are visually indistinguishable, \ie the performance is not affected by
    compression.
    Details in \cref{sec:details_large_scale_gpc}.}
  \label{fig:acc_loss_ece_over_runtimes}
  \vspace{-1ex}
\end{figure}

\begin{figure}[tbh!]
  \centering
  \includegraphics[width=0.99\linewidth]{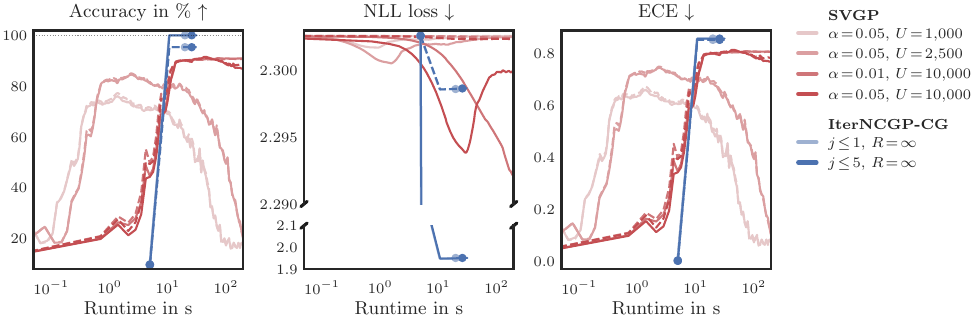}
  \caption{\textbf{GP Classification on \mnist.}
    Comparison of \svgp with learning rate $\alpha \in \{0.001, 0.01, 0.05\}$ and
    $U \in \{1000, 2500, 5000, 10000\}$ inducing points (showing only the best
    four runs) and \iterncgp-CG on a classification problem with \(\numtraindata  =
    \num[group-separator={,}]{20000}\) training points and \(\numclasses =  10\)
    classes in terms of accuracy \emph{(Left)}, NLL loss \emph{(Center)} and ECE
    \emph{(Right)}. Performance metrics are shown as solid (training set) or dashed
    (test set) lines. For \iterncgp, the dots mark the start of a new outer-loop
    iteration.
    \iterncgp-CG performs best in terms of accuracy and NLL loss but slightly worse
    in terms of ECE.
    Details in \cref{sec:details_gpc_mnist}.
    }
  \label{fig:mnist_results}
\end{figure}

\textbf{Results.}
\Cref{fig:acc_loss_ece_over_runtimes} shows the average performance of each method over
five runs that use different random seeds. Once the matrix is formed in memory, the SoD
approaches are very fast---even with $\numtraindata_\text{sub}  = 2000$, they converge
within $\SI{100}{\second}$ (all SoD runs require only two Newton steps). With increasing
$\numtraindata_\text{sub}$, the runs reach higher accuracy at the cost of increased
memory requirements. To achieve top performance, \svgp requires a large number of
inducing points (at the cost of slower training). Increasing the learning rate to
compensate results in instabilities---these runs do not exhibit competitive performance.
Both SoD and \svgp fall short of \iterncgp-CG in all three performance metrics. Using
recycling, \iterncgp-CG maintains low loss/high accuracy throughout training, even when
compression is used. It reaches the lowest final negative log-likelihood (NLL) and
expected calibration error (ECE) demonstrating better uncertainty quantification. It is
more memory-efficient than SoD (especially with compression), and, in contrast to \svgp,
does not require extensive tuning.

\textbf{Extension to \mnist.} To demonstrate \iterncgp's applicability to real-world
data, we perform a similar experiment on \mnist \citep{Lecun1998Gradient}, see
\cref{sec:details_gpc_mnist} for details. As \keops scales poorly with the data
dimension ($\inputdim = 28^2 = 784$ for \mnist), we revert to \gpytorch{}'s standard
kernel implementation, which requires more memory. We thus limit the training data to
$\numtraindata_{\text{sub}} = \num[group-separator={,}]{20000}$ images. The results
(\Cref{fig:mnist_results}) are mostly aligned with
\cref{fig:acc_loss_ece_over_runtimes}: \iterncgp-CG outperforms the well-tuned \svgp
baselines in terms of accuracy and NLL loss. Only the ECE of \iterncgp is slightly worse
than for SVGP. This is easily explained, by the fact that one can achieve smaller ECE by
accepting lower accuracy---the canonical example being a random baseline, which is
perfectly calibrated.

\section{Conclusion}
\label{sec:conclusion}

Non-conjugate Gaussian processes (NCGPs) provide a flexible probabilistic framework
encompassing, among others, GP classification and Poisson regression. Training NCGPs on
large datasets, however, necessitates approximations. Our method \iterncgp quantifies
and continuously propagates the errors caused by these approximations, in the form of
uncertainty. The information collected during training is efficiently recycled and
compressed, reducing runtime and memory requirements.

\textbf{Limitations.} A limitation of our method is directly inherited from the Laplace
approximation: The inherent error in approximating the posterior with a Gaussian is not
captured and generally depends on the choice of likelihood. However, under mild
regularity conditions, the (relative) error of the Laplace approximation scales
inversely proportional to the number of data
\citep{kass1990validity,bilodeau2022tightness}. Since we are considering the large data
regime in this work one might reasonably expect the error contribution of the Laplace
approximation itself to be small, relative to the error contribution from approximating
the MAP via Newton's method, equivalently the error contribution of approximate GP
regression.

Another limitation is the lack of a ready-to-use approach for kernel hyperparameter
estimation. However, our method is entirely composed of basic linear algebra operations
(see \cref{alg:inner,alg:outer,alg:virtual_solver_run}). Therefore one can in principle
simply differentiate with respect to the kernel hyperparameters through the entire
solve. Recent work by \citet{Wenger2024ComputationAwareGaussian} proposes an ELBO
objective to perform model selection for \itergp, which one could therefore readily
apply in our setting. We leave the open question on how to optimally choose the number
of Newton and \itergp steps per hyperparameter optimization step for future work.

\textbf{Future Work.}
So far, we have only explored the policy design space in a limited fashion. The policy
controls which areas of the data space are targeted and accounted for in the posterior.
Tailoring the actions to the \textit{specific} problem could further increase our
method's efficiency. For classification problems, a good strategy might be not to spend
compute on \datapoints where the prediction is already definitive.

Finally, a promising application for \iterncgp may be Bayesian deep learning. A popular
approach to equip a neural net with uncertainty is via a Laplace approximation
\citep{MacKay1991BayesianModel,Ritter2018ScalableLaplace,Khan2019ApproximateInference},
which is equivalent to a GP classification problem with a neural tangent kernel prior
\citep{Jacot2018NeuralTangent,Immer2021}. There, the SoD approach is regularly used
\citep[Sec.~A2.2]{Immer2021}, for which our approach might offer significant
improvements.

\section*{Acknowledgements}
The authors gratefully acknowledge co-funding by the European Union (ERC, ANUBIS,
101123955). Views and opinions expressed are however those of the author(s) only and do
not necessarily reflect those of the European Union or the European Research Council.
Neither the European Union nor the granting authority can be held responsible for them.
Philipp Hennig is a member of the Machine Learning Cluster of Excellence, funded by the
Deutsche Forschungsgemeinschaft (DFG, German Research Foundation) under Germany's
Excellence Strategy - EXC number 2064/1 - Project number 390727645; The authors further
gratefully acknowledge financial support by the DFG through Project HE 7114/5-1 in
SPP2298/1; the German Federal Ministry of Education and Research (BMBF) through the
Tübingen AI Center (FKZ:01IS18039A); and funds from the Ministry of Science, Research
and Arts of the State of Baden-Württemberg. Frank Schneider is supported by funds from
the Cyber Valley Research Fund. Lukas Tatzel is grateful to the International Max Planck
Research School for Intelligent Systems (IMPRS-IS) for support. Jonathan Wenger was
supported by the Gatsby Charitable Foundation (GAT3708), the Simons Foundation (542963),
the NSF AI Institute for Artificial and Natural Intelligence (ARNI: NSF DBI 2229929) and
the Kavli Foundation.

\bibliographystyle{tmlr}
\bibliography{bibliography}

\clearpage

\newpage
\appendix

\begin{center}
    {\Large\bf Supplementary Materials}
\end{center}

The supplementary materials contain derivations for our theoretical framework and proofs
for the mathematical statements in the main text. We also provide implementation
specifics and describe our experimental setup in more detail.

\startcontents[sections]
\printcontents[sections]{l}{1}{\setcounter{tocdepth}{3}}
\vfill
\newpage

\section{Mathematical Details}
\label{sec:math_details}

\subsection{Newton's Method as Sequential GP Regression}
\label{subsec:connection_inference_Newton_steps}

In \cref{sec:derivation_iterglm}, we reinterpret the Newton iteration as a sequence of
GP regression problems. More specifically, we rewrite the posterior predictive mean
(\cref{eq:predictive_mean}) as a GP posterior for a regression problem
(\cref{eq:inference_between_steps_mean_2}). Here, we provide a proof for this
connection.

\begin{proposition}[Reformulation of the Newton Step]
  \label{prop:reformulation_newton_step}
  Let $\mW(\vf_\newtonidx)$ be invertible. Using the transform $\vg \defeq \vf - \vm$
  and consequently $\vg_\newtonidx = \vf_\newtonidx - \vm$, the Newton step
  (\cref{eq:newton_step}) can be written as
  \begin{equation*}
    \vg_{\newtonidx + 1}
    =
    \mK (\mK + \mW(\vf_\newtonidx)^{-1})^{-1}
    \left(
      \vg_\newtonidx
      + \mW(\vf_\newtonidx)^{-1} \nabla\log p(\vy \mid \vf_\newtonidx)
    \right).
  \end{equation*}
\end{proposition}

\begin{proof}
  Recall from \cref{eq:newton_step} that
  \begin{align*}
    \vf_{\newtonidx + 1}
    =
    \vf_\newtonidx
    - \nabla^2 \Psi(\vf_\newtonidx)^{-1} \cdot \nabla\Psi(\vf_\newtonidx),
    \quad
    \text{with}
    \quad\quad
    \nabla \Psi(\vf_\newtonidx)
    &= \nabla \log p(\vy \mid \vf_\newtonidx) - \mK^{-1} (\vf_\newtonidx - \vm) \\
    \nabla^2 \Psi(\vf_\newtonidx) &= -\mW(\vf_\newtonidx) - \mK^{-1},
  \end{align*}
  where
  $
    \mW(\vf_\newtonidx)
    = -\nabla^2 \log p(\vy \mid \vf_\newtonidx)
  $
  denotes the negative Hessian (with respect to $\vf$) of the log likelihood evaluated
  at $\vf_\newtonidx$.
It holds
  \begin{align*}
    \vf_{\newtonidx + 1}
    &= \vf_\newtonidx
    - \nabla^2 \Psi(\vf_\newtonidx)^{-1} \cdot \nabla\Psi(\vf_\newtonidx) \\
    &= \vf_\newtonidx + (\mW(\vf_\newtonidx) + \mK^{-1})^{-1} \cdot
    \left(
      \nabla\log p(\vy \mid \vf_\newtonidx) - \mK^{-1} (\vf_\newtonidx - \vm)
    \right) \\
  \intertext{By substracting $\vm$ from both sides we obtain}
    \vg_{\newtonidx + 1}
    &= \vg_\newtonidx + (\mW(\vf_\newtonidx) + \mK^{-1})^{-1} \cdot
    \left( \nabla \log p(\vy \mid \vf_\newtonidx) - \mK^{-1} \vg_\newtonidx \right) \\
    &= (\mW(\vf_\newtonidx) + \mK^{-1})^{-1}
    \left(
      (\mW(\vf_\newtonidx) + \mK^{-1}) \vg_\newtonidx
      + \nabla\log p(\vy \mid \vf_\newtonidx) - \mK^{-1} \vg_\newtonidx
    \right) \\
    &= (\mW(\vf_\newtonidx) + \mK^{-1})^{-1}
    \left(
      \mW(\vf_\newtonidx)\vg_\newtonidx + \nabla\log p(\vy \mid \vf_\newtonidx)
    \right) \\
    &= (\mW(\vf_\newtonidx) + \mK^{-1})^{-1} \mW(\vf_\newtonidx)
    \left(
      \vg_\newtonidx
      + \mW(\vf_\newtonidx)^{-1} \nabla\log p(\vy \mid \vf_\newtonidx)
    \right) \\
    &= (\mI + \mW(\vf_\newtonidx)^{-1} \mK^{-1})^{-1}
    \left(
      \vg_\newtonidx
      + \mW(\vf_\newtonidx)^{-1} \nabla\log p(\vy \mid \vf_\newtonidx)
    \right) \\
    &= (\mK \mK^{-1} + \mW(\vf_\newtonidx)^{-1} \mK^{-1})^{-1}
    \left(
      \vg_\newtonidx
      + \mW(\vf_\newtonidx)^{-1} \nabla\log p(\vy \mid \vf_\newtonidx)
    \right) \\
    &= \mK (\mK + \mW(\vf_\newtonidx)^{-1})^{-1}
    \left(
      \vg_\newtonidx
      + \mW(\vf_\newtonidx)^{-1} \nabla\log p(\vy \mid \vf_\newtonidx)
    \right) \\
    &= \mK \KWinvf{\vf_\newtonidx}^{-1}
    \left(
      \vg_\newtonidx
      + \mW(\vf_\newtonidx)^{-1} \nabla\log p(\vy \mid \vf_\newtonidx)
    \right),
  \end{align*}
  with $\KWinvf{\vf_\newtonidx} \defeq \mK + \mW(\vf_\newtonidx)^{-1}$. \qedhere
\end{proof}

\textbf{Newton's Method as Sequential GP Regression.} Using the LA at $\vf_i$, we obtain a
GP posterior (see \cref{eq:predictive_mean,eq:predictive_variance} in
\cref{sec:background}). With \cref{prop:reformulation_newton_step} (\ie assuming that
$\mW(\vf_\newtonidx)^{-1}$ exists), we can rewrite \cref{eq:predictive_mean} as
\begin{align*}
    m_{\newtonidx,*}(\cdot)
    &= m(\cdot)
    + K(\cdot, \mX) \mK^{-1} (\vf_{\newtonidx + 1} - \vm) \\
  &= m(\cdot)
    + K(\cdot, \mX) \mK^{-1} \vg_{\newtonidx + 1} \\
  &= m(\cdot)
    + K(\cdot, \mX) \mK^{-1} \mK
  \KWinvf{\vf_i}^{-1}
  \left(
    \vg_\newtonidx
    + \mW(\vf_\newtonidx)^{-1} \nabla\log p(\vy \mid \vf_\newtonidx)
  \right) \\
  &= m(\cdot)
    + K(\cdot, \mX)
  \KWinvf{\vf_\newtonidx}^{-1}
  \bigl(
    \vf_\newtonidx
    + \mW(\vf_\newtonidx)^{-1} \nabla\log p(\vy \mid \vf_\newtonidx) - \vm
  \bigr) \\
  &= m(\cdot)
    + K(\cdot, \mX)
  \KWinvf{\vf_\newtonidx}^{-1}
  (\pseudotargetsf{\vf_\newtonidx} - \vm),
\end{align*}
where $\pseudotargetsf{\vf_\newtonidx} \defeq \vf_\newtonidx + \mW(\vf_\newtonidx)^{-1}
\nabla\log p(\vy \mid \vf_\newtonidx)$. This proves
\cref{eq:inference_between_steps_mean_2}. Together with \cref{eq:predictive_variance},
$m_{\newtonidx,*}$ defines a GP posterior for a GP regression problem with pseudo
targets $\pseudotargetsf{\vf_\newtonidx}$ observed with Gaussian noise
$\gaussian{\vzero}{\mW(\vf_\newtonidx)^{-1}}$ \citep[Eqs. (2.24) and
(2.38)]{Rasmussen2006}.

\cref{eq:inference_between_steps_mean_2} requires solving the linear system
$\KWinvf{\vf_\newtonidx} \cdot \vv = \pseudotargetsf{\vf_\newtonidx} - \vm$ of size
$\numtraindata \numclasses$. Then, $m_{\newtonidx,*}(\cdot) = m(\cdot) +
\kernelfn(\cdot, \mX) \vv$. In \cref{prop:reformulation_newton_step}, we can write
$\vg_{\newtonidx + 1}$ as $\vg_{\newtonidx + 1} = \mK \vv$, \ie $\vf_{\newtonidx + 1} =
\mK \vv + \vm$. So, \textit{both} the predictive mean $m_{\newtonidx,*}$ \textit{and}
the Newton update $\vf_{\newtonidx + 1}$ follow directly from the solution $\vv$. In
that sense, performing inference and computing Newton iterates are equivalent.

\textbf{What If $\mW(\vf_\newtonidx)$ Is \textit{Not} Invertible?} For multi-class
classification, \(\mW\) has rank \(\numtraindata(\numclasses-1)\) and thus $\mW^{-1}$
does not exist. Therefore, we use its pseudo-inverse $\mW^\pinv$ instead. We derive an
explicit expression for $\mW^\pinv$ in \cref{subsec:pseudo_inverse_w} which allows for
efficient matrix-vector multiplies.

\subsection{Our Algorithm is an Extension of \itergp}
\label{subsec:extension_of_itergp}

Our algorithm \iterncgp uses \itergp as a core building block. \iterncgp{}'s outer loop
(\cref{alg:outer}) can be understood as a sequence of GP regression problems and we use
\itergp (that implements the inner loop, see \cref{alg:inner}) for finding approximate
solutions to each of these problems. In the case of GP regression (\ie with a Gaussian
likelihood), the outer loop collapses to a \textit{single} iteration and \iterncgp
coincides exactly with \itergp, as we show in the following.

\begin{theorem}[Generalization of \itergp]
  \label{theo:extension_of_itergp}
  For a Gaussian likelihood $p(\vy \mid \vf) = \gaussianpdf{\vy}{\vf}{\mLambda}$,
  \iterncgp converges in a single Newton step (\ie $\vf_1 = \vf_\text{MAP}$) and \iterncgp
  (\cref{alg:outer}) coincides exactly with \itergp (\cref{alg:inner}).
\end{theorem}

\begin{proof}
Since the likelihood is Gaussian
$p(\vy \mid \vf) = \gaussianpdf{\vy}{\vf}{\mLambda}$,
the log likelihood is given by
\begin{equation*}
  \log p(\vy \mid \vf)
  \eqc
  -\frac{1}{2} (\vf - \vy)^\top \mLambda^{-1} (\vf - \vy).
\end{equation*}
This gives rise to a log-posterior (\cref{eq:Psi})
\begin{align*}
  \Psi(\vf)
  \defeq&
  \log p(\vf \mid \mX, \vy) \\
  \eqc&
  \log p(\vy \mid \vf)
  - \frac{1}{2} (\vf - \vm)^\top \mK^{-1} (\vf - \vm) \\
  =&
  -\frac{1}{2} (\vf - \vy)^\top \mLambda^{-1} (\vf - \vy)
  - \frac{1}{2} (\vf - \vm)^\top \mK^{-1} (\vf - \vm)
\end{align*}
that is \textit{quadratic} in $\vf$. The first Newton iterate $\vf_1$ therefore
coincides with log-posterior's maximizer $\vf_1 = \vf_\text{MAP}$. The outer loop of
\iterncgp thus reduces to a single iteration.

How does this step look from the perspective of the \iterncgp algorithm? First note that
$\mW(\vf) = -\nabla^2 \log p(\vy \mid \vf) \equiv \mLambda^{-1}$. Given an initial
$\vf_0$, \iterncgp computes the observation noise $\mW^{-1}(\vf_0) = \mLambda$ and pseudo
regression targets
$
  \pseudotargetsf{\vf_0}
  = \vf_0 + \mW(\vf_0)^{-1} \nabla\log p(\vy \mid \vf_0)
  = \vf_0 - \mLambda \mLambda^{-1} (\vf_0 - \vy)
  = \vy.
$
Both these quantities are \textit{independent} of the initialization $\vf_0$. Thus,
the first (and only) regression problem that \iterncgp forms in the outer loop is the
\textit{original} regression problem (defined by labels $\vy$ and the observation noise
$\Lambda$) and \itergp is applied to solve that regression problem. This shows that our
framework recovers \itergp in the case of a Gaussian likelihood and our algorithm can
thus be regarded as an extension thereof.
\end{proof}

\subsection{The Marginal Uncertainty Decreases in the Inner Loop}
\label{subsec:uncertainty_decreases_inner_loop}

We claim in \cref{sec:derivation_iterglm} that the marginal uncertainty captured by
$\kernelfn_{\newtonidx, \solveridx}(\vx, \vx) \in \R^{\numclasses \times \numclasses}$
(see \cref{eq:inference_PLS_iteration_covar}) within a Newton step decreases with each
solver iteration $\solveridx$. Here, we provide the proof for this statement.

\begin{proposition}[The Uncertainty Decreases in the Inner Loop]
For each $i$ it holds (element-wise) that
$
  \diag(\kernelfn_{\newtonidx, \solveridx}(\vx, \vx))
  \geq
  \diag(\kernelfn_{\newtonidx, k}(\vx, \vx))
$
for any $k \geq \solveridx$ and arbitrary $\vx$.

\proof To see this, we rewrite $\mC_\solveridx$ as a sum of $\solveridx$ rank-$1$
matrices $\mC_\solveridx = \sum_{\ell=1}^\solveridx \vd_\ell \vd_\ell^\top$ and
substitute this into \cref{eq:inference_PLS_iteration_covar}. It holds that
\begin{align*}
  \diag(\kernelfn_{\newtonidx, \solveridx}(\vx, \vx))
  &=
  \diag(\kernelfn(\vx, \vx))
  - \sum_{\ell=1}^j
  \diag(
    \underbracket[0.1ex]{\kernelfn(\vx, \mX) \vd_\ell}_{\eqdef \hat{\vd}_\ell}
    \underbracket[0.1ex]{\vd_\ell^\top \kernelfn(\mX, \vx)}_{= \hat{\vd}_\ell^\top}
  )\\
  &=
  \diag(\kernelfn(\vx, \vx))
  - \sum_{\ell=1}^j \hat{\vd}_\ell^{\,2}
  \tag*{The square is applied element-wise.}\\
  &\geq
  \diag(\kernelfn(\vx, \vx))
  - \sum_{\ell=1}^k \hat{\vd}_\ell^{\,2}
  \quad \text{for $k \geq \solveridx$}
  \\
  &=
  \diag(\kernelfn_{\newtonidx, k}(\vx, \vx))
\end{align*}
for any $\vx \in \inputspace$. \qedhere
\end{proposition}

\subsection{Virtual Solver Run}
\label{sec:details_virtual_solver_run}

In \cref{sec:recycling}, we showed that it is possible to \textit{imitate} a solver run
using the \textit{previous} actions on the \textit{new} problem, without ever having to
multiply by $\mK$. The pseudo code is given in \cref{alg:virtual_solver_run}. Here, we
discuss the numerical and probabilistic perspective on that procedure in more detail and
provide derivations for the statements in the main text.

\textbf{Numerical Perspective.}
Let $\mS = (\action_1, \dots, \action_\buffersize)$ the matrix of stacked linearly
independent actions. We use $\mC_0 = \mS (\mS^\top \KWinv \mS)^{-1} \mS^\top$ (see
\cref{eq:C_j}) as an initial estimate of the precision matrix in \cref{alg:inner}. The
corresponding initial residual (see \cref{alg:inner}) $\vr_0 = (\pseudotargets - \vm) -
\KWinv \vv_0$ for the first iterate $\vv_0 = \mC_0 (\pseudotargets - \vm)$ can be
decomposed into $\mP_{\mS} \vr_0$ and $(\mI - \mP_{\mS}) \vr_0$. $\mP_{\mS} = \mS
(\mS^\top \mS)^{-1} \mS^\top$ is the orthogonal projection onto the subspace
$\vspan\{\mS\}$ spanned by the actions.

\begin{proposition}[Residual in $\vspan\{\mS\}$ Is Zero]
  \label{prop:projected_residual_zero}
  The orthogonal projection $\mP_{\mS} \vr_0$ of the initial residual $\vr_0$ onto
  $\vspan\{\mS\}$ is zero.

  \proof It holds that
  \begin{align*}
    \mP_{\mS} \vr_0
    &=
    \mP_{\mS} (\pseudotargets - \vm) - \mP_{\mS} \KWinv \vv_0 \\
    &=
    \mP_{\mS} (\pseudotargets - \vm)
    - \mP_{\mS} \KWinv \mC_0 (\pseudotargets - \vm) \\
    &=
    \mP_{\mS} (\pseudotargets - \vm)
    - \underbracket[0.1ex]{
      \mS(\mS^\top \mS)^{-1} (\mS^\top
    }_{=\mP_{\mS}}
    \KWinv
    \underbracket[0.1ex]{
      \mS) (\mS^\top \KWinv \mS)^{-1} \mS^\top
    }_{=\mC_0} (\pseudotargets - \vm) \\
    &= \mP_{\mS} (\pseudotargets - \vm)
    - \underbracket[0.1ex]{\mS (\mS^\top \mS)^{-1} \mS^\top}_{= \mP_{\mS}}
    (\pseudotargets - \vm) \\
    &= 0.\qedhere
  \end{align*}
\end{proposition}

\cref{prop:projected_residual_zero} shows that the residual in $\vspan\{\mS\}$ is zero.
In that sense, the solution within this subspace is already perfectly identified at
initialization. The remaining residual thus lies in the orthogonal complement of
$\vspan\{\mS\}$ which can be targeted through additional actions. If we measure the
error in the representer weights $\vv - \vv_0$, a similar results holds, as we
show in the following.

\begin{proposition}[Error in Representer Weights in $\vspan\{\mS\}$ Is Zero]
  The $\KWinv$-orthogonal projection of the representer weights approximation error
   $\hat{\mP}_{\mS} (\vv - \vv_0)$  onto $\vspan\{\mS\}$ is zero.

  \proof The $\KWinv$-orthogonal (orthogonal with respect to the inner product
  $\langle\cdot,\cdot\rangle_{\KWinv}$) projection onto the subspace $\vspan\{\mS\}$
  spanned by the actions is given by $\hat{\mP}_{\mS} = \mC_0 \KWinv$ \citep[Section
  S2.1]{Wenger2022PosteriorComputational}. It holds that
  \begin{align*}
    \hat{\mP}_{\mS} (\vv - \vv_0)
    &= \mC_0 \KWinv (\vv - \vv_0)\\
    &= \mC_0 \KWinv \KWinv^{-1} (\pseudotargets - \vm)
    - \mC_0 \KWinv \mC_0 \underbracket[0.1ex]{(\pseudotargets - \vm)}_{=\KWinv \vv}\\
    &= \mC_0 (\pseudotargets - \vm)
    - \underbracket[0.1ex]{\mC_0 \KWinv}_{=\hat{\mP}_{\mS}}
    \underbracket[0.1ex]{\mC_0 \KWinv}_{=\hat{\mP}_{\mS}} \vv\\
    &= \mC_0 (\pseudotargets - \vm)
    - \mC_0 \underbracket[0.1ex]{\KWinv \vv}_{=\pseudotargets - \vm} \\
    &= 0,
  \end{align*}
  where we used that $\vv = \KWinv^{-1} (\pseudotargets - \vm)$ is the solution of the
  GP regression linear system, $\vv_0 = \mC_0 (\pseudotargets - \vm)$ and the
  idempotence of the projection matrix $\hat{\mP}_{\mS} = \hat{\mP}_{\mS}
  \hat{\mP}_{\mS}$.\qedhere
\end{proposition}

\textbf{Probabilistic Perspective.}
\cref{eq:inference_PLS_iteration_covar} describes the effect of $\mC_0$ from a
probabilistic perspective.
Initializing $\mC_0 = \mzero$ in step $\newtonidx$ results in $m_{\newtonidx,0} =
m(\cdot)$ (prior mean) and $\kernelfn_{\newtonidx,0} = \kernelfn(\cdot, \cdot)$ (prior
covariance) since the \textit{reduction} of uncertainty $\kernelfn(\cdot, \mX) \, \mC_0
\, \kernelfn(\mX, \cdot)$ is zero. This case, where no information from past steps is
included, is illustrated in the first column $\bufferlimit = 0$ in
\cref{fig:virtual_solver_run}.

\textbf{Special Case.} We consider a special case, where the general intricate form of
the total marginal variance from \cref{sec:compression}
\begin{align}
  \Tr(\kernelfn_{\newtonidx,0}(\mX, \mX))
  &=
  \Tr(\mK) - \Tr(\mK \mC_0 \mK)
  \label{eq:total_marginal_variance}
\end{align}
collapses. Let $\lambda_1,
\dots, \lambda_{\numtraindata\numclasses} > 0$ denote the eigenvalues of $\KWinv$ and
$\vb_1, \dots, \vb_{\numtraindata\numclasses}$ the corresponding (pairwise orthogonal)
eigenvectors. We make the following two assumptions: \textbf{(A1):} We assume $\mW^{-1}
= \mzero$, \ie $\KWinv = \mK$. \textbf{(A2):} We assume that the actions coincide with a
subset $\sL \subseteq \{1, \dots, \numtraindata\numclasses\}$ of $\KWinv$'s eigenvectors
$\mS = (\vb_l)_{l \in \sL} \in \R^{\numtraindata\numclasses \times \vert\sL\vert}$.

\begin{proposition}[Total Marginal Uncertainty]
  \label{prop:deflation_uncertainty}
  Under assumptions \textnormal{\textbf{(A1)}} and \textnormal{\textbf{(A2)}} it holds that
    \begin{equation*}
        \Tr(\kernelfn_{\newtonidx,0}(\mX, \mX))
        =
        \Tr(\mK) - \Tr(\mM).
    \end{equation*}

  \proof Let
    $
      \mS
      = (\vb_l)_{l \in \sL} \in \R^{\numtraindata\numclasses \times \vert\sL\vert}
    $
    and
    $
        \mLambda
        = \diag((\lambda_l)_{l \in \sL})
        \in \R^{\vert\sL\vert \times \vert\sL\vert}
    $
    contain a subset $\sL \subseteq \{1, \dots, \numtraindata\numclasses\}$ of
    $\KWinv$'s eigenpairs. The remaining eigenvectors and eigenvalues are given by
    $
        \mS_+
        = (\vb_l)_{l \notin \sL} \in
        \R^{\numtraindata\numclasses \times \numtraindata\numclasses - \vert\sL\vert}
    $
    and
    $
        \mLambda_+
        = \diag((\lambda_l)_{l \notin \sL})
        \in \R^{\numtraindata\numclasses - \vert\sL\vert
        \times \numtraindata\numclasses - \vert\sL\vert}
    $.
    First note that we can write the eigendecomposition of $\KWinv = \mK$ as a sum of
    two components
    $
        \KWinv = \mS \mLambda \mS^\top + \mS_+ \mLambda_+ \mS_+^\top,
    $
    each of which covers one part of the spectrum. It holds
    \begin{equation*}
        \mS^\top \mS = \mI,
        \quad
        \mS_+^\top \mS_+ = \mI,
        \quad
        \mS^\top \mS_+ = \mzero
        \quad \text{and}
        \quad \mS_+^\top \mS = \mzero
    \end{equation*}
    since $\KWinv$ is symmetric and its eigenvectors are thus pairwise orthogonal. It
    follows
    \begin{align*}
        \mK \mS
        &= (\mS \mLambda \mS^\top + \mS_+ \mLambda_+ \mS_+^\top) \mS
        = \mS \mLambda \\
        \mS^\top \mK
        &= \mS^\top (\mS \mLambda \mS^\top + \mS_+ \mLambda_+ \mS_+^\top)
        = \mLambda \mS^\top \\
        \mM
        = \mS^\top \mK \mS
        &= \mS^\top (\mS \mLambda \mS^\top + \mS_+ \mLambda_+ \mS_+^\top) \mS
        = \mLambda.
    \end{align*}
    Plugging those expressions into \cref{eq:total_marginal_variance} yields
    \begin{align*}
        \Tr(\kernelfn_{\newtonidx,0}(\mX, \mX))
        &=
        \Tr(\mK) - \Tr(\mK \mC_0 \mK) \\
        &=
      \Tr(\mK) - \Tr(\mK \mS
      \underbracket[0.1ex]{(\mS^\top \KWinv \mS)^{-1}}_{=\mM^{-1}}
      \mS^\top \mK) \\
        &=
        \Tr(\mK)
        - \Tr(\mS \mLambda \mLambda^{-1} \mLambda \mS^\top) \\
        &=
        \Tr(\mK)
        - \Tr(\mS^\top \mS \mLambda) \\
        &=
        \Tr(\mK) - \Tr(\mLambda) \\
        &=
        \Tr(\mK) - \Tr(\mM) \\
        &=
        \sum_{l \notin \sL} \lambda_l.
    \end{align*}
\end{proposition}
The last equation is due to
\begin{equation*}
    \Tr(\mK)
    = \Tr(\mS \mLambda \mS^\top) + \Tr(\mS_+ \mLambda_+ \mS_+^\top)
    = \Tr(\mS^\top \mS \mLambda) + \Tr(\mS_+^\top \mS_+ \mLambda_+)
    = \Tr(\mLambda) + \Tr(\mLambda_+).\qedhere
\end{equation*}

\cref{prop:deflation_uncertainty} shows that the \textit{reduction} of the marginal
uncertainty is determined by the sum of $\mM$'s eigenvalues $\sum_{l \in \sL}
\lambda_l$. If $\mS$ contains the eigenvectors $\vb_l$ to the largest eigenvalues (\ie
$\mS$ is ``aligned'' with the high-variance subspace of $\KWinv$), the remaining
uncertainty $\sum_{l \notin \sL} \lambda_l$ is small. In contrast, if $\mS$ covers the
low-variance subspace of $\KWinv$, the uncertainty remains largely unaffected.

\subsection{Derivatives of the Poisson Log Likelihood}
\label{sec:math_poisson_regression}

One of our main experiments in \cref{sec:experiments} is Poisson regression (see
\cref{sec:details_poisson_regression} for details). In order to apply \iterncgp, we have
to formulate the problem within the NCGP framework. In particular, we have to specify
the derivatives of the log likelihood.

The Poisson likelihood is given by
\begin{equation*}
    p(\vy \mid \vf)
    = \prod_{n=1}^\numtraindata
    \frac{\lambda_n^{y_n} \exp(-\lambda_n)}{y_n!},
\end{equation*}
where $y_n \in \N_0$ and $\vlambda \defeq \lambda(\mX) = \exp(f(\mX)) = \exp(\vf)$.
Taking the logarithm yields
\begin{equation*}
    \log p(\vy \mid \vf)
    = \sum_{n=1}^\numtraindata
  \log\left(
    \frac{\lambda_n^{y_n} \exp(-\lambda_n)}{y_n!}
  \right)
  = \sum_{n=1}^\numtraindata
  \left(
    y_n \log(\lambda_n) - \lambda_n - \log(y_n!)
  \right).
\end{equation*}

The log likelihood's gradient and Hessian with respect to $\vf$ are therefore given by
\begin{equation*}
  \nabla \log p(\vy \mid \vf) = \vy - \exp(\vf)
  \quad \text{and} \quad
  \nabla^2 \log p(\vy \mid \vf) = - \diag(\exp(\vf)),
\end{equation*}
where the exponential function is applied element-wise. This implies that the log
likelihood is concave which was one of the prerequisites of our algorithm (see
\cref{sec:ncgps}). It follows that $\mW(\vf)^{-1} = \diag(\exp(-\vf))$.

\subsection{Pseudo-Inverse of Negative Hessian of the Log Likelihood for Multi-Class Classification}
\label{subsec:pseudo_inverse_w}

For multi-class classification (see \cref{sec:details_large_scale_gpc} for details), we
need access to the pseudo inverse $\mW^\pinv$. For this to be efficient, we derive an
explicit form of $\mW^\pinv$ in the following and show that matrix-vector multiplies can
be implemented efficiently in $\bigO{\numtraindata\numclasses}$. Since the ordering (see
\cref{sec:ordering}) of $\mW$ plays an important role in the derivation, we use an
explicit notation in this section.

\begin{lemma}[Explicit Pseudo-Inverse for Multi-Class Classification]
Consider multi-class classification, such that the log likelihood \(\log p(\vy \mid
\vf)\) is given by a categorical likelihood with a softmax inverse link function, then
the pseudoinverse $[\mW(\vf)]_{CN}^{\pinv} \in \R^{\numtraindata\numclasses \times
\numtraindata\numclasses}$ of \(\mW(\vf)\) in \cnordering{} is given by
\begin{equation*}
  [\mW(\vf)]_{CN}^{\pinv} =
  \begin{pmatrix}
    (\mI - \frac{1}{\numclasses} \vone \vone^\top)\diag(\vpi_1^{-1})(\mI - \frac{1}{\numclasses}\vone\vone^\top) & &\\
    & \ddots &\\
    &  & (\mI - \frac{1}{\numclasses} \vone \vone^\top)\diag(\vpi_\numtraindata^{-1})(\mI - \frac{1}{\numclasses}\vone\vone^\top)
  \end{pmatrix} ,
\end{equation*}
where \(\vpi_n = (\pi_n^1, ..., \pi_n^\numclasses)^\top \in \R^{\numclasses}\) denotes
the output of the softmax for $\vx_n$, \ie $\pi_n^c \defeq \exp(f_n^c) / \sum_{c'}
\exp(f_n^{c'})$. The cost of one matrix-vector multiplication \(\vv \mapsto
[\mW(\vf)]_{CN}^{\pinv} \vv\) with the pseudo-inverse is \(\bigO{\numtraindata
\numclasses}\).

\end{lemma}

\begin{proof}
  By Eq. (3.38) in \citet{Rasmussen2006}, the matrix \(\mW(\vf)\) in \ncordering\ is given by
  \begin{equation*}
    [\mW(\vf)]_{\numtraindata\numclasses}= [\diag(\vpi)]_{\numtraindata\numclasses} - \mPi \mPi^\top,
  \end{equation*}
  where \([\diag(\vpi)]_{\numtraindata\numclasses} = \diag(\pi_1^1, \dots, \pi_\numtraindata^1, \dots, \pi^\numclasses_1, \dots \pi^\numclasses_\numtraindata)\) and
  \begin{equation*}
     \mPi = \begin{pmatrix} \diag(\pi_1^1, \dots, \pi_\numtraindata^1) \\ \vdots \\ \diag(\pi_1^\numclasses, \dots, \pi_\numtraindata^\numclasses)\end{pmatrix} \in \R^{\numtraindata\numclasses \times \numtraindata}.
  \end{equation*}
  Rewriting \(\mW(\vf)\) in the \cnordering, we obtain using $[\diag(\vpi)]_{\numclasses\numtraindata} = \diag(\pi_1^1, \dots, \pi_1^\numclasses, \dots,\pi^1_\numtraindata, \allowbreak \dots, \pi^\numclasses_\numtraindata)$  that
  \begin{align*}
    [\mW(\vf)]_{\numclasses\numtraindata}
    &=  [\diag(\vpi)]_{\numclasses\numtraindata} - \begin{pmatrix} \vpi_1 & & \\ & \ddots &\\ & & \vpi_\numtraindata\end{pmatrix} \begin{pmatrix} \vpi_1^\top & & \\ & \ddots &\\ & & \vpi_\numtraindata^\top \end{pmatrix} \\
    &= \blockdiag(\diag(\vpi_n) - \vpi_n \vpi_n^\top).
  \end{align*}
  Now the pseudoinverse of a block-diagonal matrix is the block-diagonal of the block pseudoinverses, \ie
  $\blockdiag(\mA_n)^\pinv = \blockdiag(\mA_n^\pinv)$
  which can be shown by simply checking the definition criteria of the pseudo-inverse and using basic properties of block matrices. Therefore it suffices to show that the block pseudoinverses are given by
  \begin{equation*}
    (\diag(\vpi_n) - \vpi_n \vpi_n^\top)^\pinv = (\mI - \frac{1}{\numclasses} \vone \vone^\top)\diag(\vpi_n^{-1})(\mI - \frac{1}{\numclasses}\vone\vone^\top)
  \end{equation*}
  for \(n \in \{1, \dots, \numtraindata\}\). We do so by checking the definition criteria of a pseudoinverse. Let \(\mA_n = \diag(\vpi_n) - \vpi_n \vpi_n^\top\). We begin by showing the following intermediate result:
  \begin{equation}
    \label{eqn:pseudoinv_classification_nntermediate_result}
     \mA_n (\mI - \frac{1}{\numclasses}\vone\vone^\top) =  \mA_n - \frac{1}{\numclasses} (\diag(\vpi_n) - \vpi_n \vpi_n^\top) \vone \vone^\top= \mA_n - \frac{1}{\numclasses}(\vpi_n - \vpi_n (\vpi_n^\top \vone))\vone^\top = \mA_n.
  \end{equation}
  Now let's verify the first criterion in the definition of the pseudoinverse. We have
  \begin{align*}
    \mA_n (\mI - \frac{1}{\numclasses} \vone \vone^\top)\diag(\vpi_n^{-1})(\mI - \frac{1}{\numclasses}\vone\vone^\top) \mA_n &= \mA_n \diag(\vpi_n^{-1}) \mA_n\\
    &= \mA_n \diag(\vpi_n^{-1})(\diag(\vpi_n) - \vpi_n \vpi_n^\top)\\ &= \mA_n (\mI - \vone \vpi_n^\top)\\
    &= \mA_n - (\diag(\vpi_n) - \vpi_n \vpi_n^\top)\vone \vpi_n^\top\\
    &= \mA_n,
  \end{align*}
  where we used \eqref{eqn:pseudoinv_classification_nntermediate_result}. Next, we'll verify the second criterion.
  \begin{align*}
    (\mI - \frac{1}{\numclasses} \vone \vone^\top)\diag(\vpi_n^{-1})(\mI - \frac{1}{\numclasses}\vone\vone^\top) &\mA_n (\mI - \frac{1}{\numclasses} \vone \vone^\top)\diag(\vpi_n^{-1})(\mI - \frac{1}{\numclasses}\vone\vone^\top)\\
     &= (\mI - \frac{1}{\numclasses} \vone \vone^\top)\diag(\vpi_n^{-1})\mA_n\diag(\vpi_n^{-1})(\mI - \frac{1}{\numclasses}\vone\vone^\top)\\
     &= (\mI - \frac{1}{\numclasses} \vone \vone^\top)(\vpi_n^{-1})(\mI - \frac{1}{\numclasses}\vone\vone^\top)
  \end{align*}
  where we used
  \begin{equation}
    \label{eqn:pseudoinv_classification_nntermediate_result2}
    \diag(\vpi_n^{-1})\mA_n = \mI = \mA_n \diag(\vpi_n^{-1})
  \end{equation}
  as shown above. Finally, we verify the symmetry of the product of \(\mA_n\) and its pseudoinverse. Observe that both \(\mA_n\) and \((\mI - \frac{1}{\numclasses} \vone \vone^\top)\diag(\vpi_n^{-1})(\mI - \frac{1}{\numclasses}\vone\vone^\top)\) are symmetric. Therefore we have
  \begin{align*}
    (\mA_n (\mI - \frac{1}{\numclasses} \vone \vone^\top)\diag(\vpi_n^{-1})(\mI - \frac{1}{\numclasses}\vone\vone^\top))^* = (\mI - \frac{1}{\numclasses} \vone \vone^\top)\diag(\vpi_n^{-1})(\mI - \frac{1}{\numclasses}\vone\vone^\top) \mA_n\\
    \intertext{and}
    ((\mI - \frac{1}{\numclasses} \vone \vone^\top)\diag(\vpi_n^{-1})(\mI - \frac{1}{\numclasses}\vone\vone^\top)\mA_n)^* = \mA_n (\mI - \frac{1}{\numclasses} \vone \vone^\top)\diag(\vpi_n^{-1})(\mI - \frac{1}{\numclasses}\vone\vone^\top).
  \end{align*}
  Thus if we can show that \(\mA_n\) and \((\mI - \frac{1}{\numclasses} \vone \vone^\top)\diag(\vpi_n^{-1})(\mI - \frac{1}{\numclasses}\vone\vone^\top)\) commute we have shown the remaining symmetry criteria of the pseudoinverse. It holds that
  \begin{align*}
    \mA_n (\mI - \frac{1}{\numclasses} \vone \vone^\top)\diag(\vpi_n^{-1})(\mI - \frac{1}{\numclasses}\vone\vone^\top) &\overset{\eqref{eqn:pseudoinv_classification_nntermediate_result}}{=} \mA_n \diag(\vpi_n^{-1})(\mI - \frac{1}{\numclasses} \vone \vone^\top)\overset{\eqref{eqn:pseudoinv_classification_nntermediate_result2}}{=} (\mI - \frac{1}{\numclasses} \vone \vone^\top)
    \intertext{as well as}
    (\mI - \frac{1}{\numclasses} \vone \vone^\top)\diag(\vpi_n^{-1})(\mI - \frac{1}{\numclasses}\vone\vone^\top) \mA_n &\overset{\eqref{eqn:pseudoinv_classification_nntermediate_result}}{=} (\mI - \frac{1}{\numclasses} \vone \vone^\top)\diag(\vpi_n^{-1}) \mA_n \overset{\eqref{eqn:pseudoinv_classification_nntermediate_result2}}{=}(\mI - \frac{1}{\numclasses} \vone \vone^\top)
  \end{align*}
  This completes the proof for the form of the pseudoinverse. For the complexity of multiplication, note that multiplying with \((\mI - \frac{1}{\numclasses} \vone \vone^\top)\diag(\vpi_n^{-1})(\mI - \frac{1}{\numclasses}\vone\vone^\top)\) has cost \(\bigO{\numclasses}\), since it decomposes into two multiplications with \( (\mI - \frac{1}{\numclasses}\vone\vone^\top)\) which is linear and one elementwise scaling. Therefore the cost of multiplication with the pseudoinverse consisting of \(\numtraindata\) blocks has complexity \(\bigO{\numtraindata\numclasses}\).
\end{proof}

\section{Implementation Details}
\label{sec:implementation_details}

\subsection{Ordering within Vectors \& Matrices}
\label{sec:ordering}

\textbf{Ordering within Vectors.} By default, we assume all vectors and matrices to be
represented in \cnordering. For example, the mean vector was introduced as the
aggregated outputs of the mean function $m \colon \inputspace \to \R^\numclasses$ for
all \datapoints $\vm = m(\mX) = (m(\vx_1)^\top, \dots, m(\vx_\numtraindata)^\top)^\top$.
With $m(\vx_n)^\top = (m_n^1, \dots, m_n^\numclasses)$ denoting the $\numclasses$
outputs for \datapoint $\vx_n$, we can write $\vm$ as
$
  \vm
   = (m_1^1, m_1^2, \dots, m_1^\numclasses,
    \dots,
    m_\numtraindata^1, m_\numtraindata^2, \dots, m_\numtraindata^\numclasses).
$
We call that representation \cnordering because the superscript $c$ moves
\textit{first} and the subscript $n$ moves \textit{second}. Consecutively,
$
  (m_1^1, m_2^1, \dots, m_\numtraindata^1,
    \dots,
    m_1^\numclasses, m_2^\numclasses, \dots, m_\numtraindata^\numclasses)
$
corresponds to \ncordering.

\textbf{Ordering within Matrices.} The same terminology can be applied to matrices,
where the rows and columns can be represented in $\numclasses\numtraindata$ or
\ncordering. Depending on the context, different representations are beneficial. For
example, in \cnordering, $\mW$ is block-diagonal (due to our iid assumption, see
\cref{sec:ncgps}) with $\numtraindata$ blocks of size $\numclasses \times \numclasses$ on
the diagonal. In contrast, when the $\numclasses$ outputs of the hidden function are
assumed to be independent of each other, $\mK$ is block diagonal only in \ncordering.
So, based on the chosen ordering, different structures arise that we can exploit in
subsequent computations (\eg when we compute the inverse of $\mW$, see
\cref{sec:details_cost_analysis}).

\subsection{Stopping Criteria in \cref{alg:inner,alg:outer}}

\textbf{Stopping Criterion in \cref{alg:outer}.} The \textsc{OuterStoppingCriterion}()
we use for our experiments is based on the \textit{relative change} of the vector
$\vg_\newtonidx = \vf_\newtonidx - \vm$. When $\Vert \vg_\newtonidx - \vg_{\newtonidx -
1} \Vert \, \Vert \vg_\newtonidx \Vert^{-1} \leq \convtol$ falls below the convergence
tolerance $\convtol$ (by default, $\convtol = 1 \%$), the loop over $\newtonidx$
terminates. Of course, other convergence criteria are also conceivable. Depending on the
application one might want to customize the criterion and, for example, include the
marginal uncertainty at the training data.

\textbf{Stopping Criterion in \cref{alg:inner}.} We use the same
\textsc{InnerStoppingCriterion}() as in \citep[Section
S3.2]{Wenger2022PosteriorComputational}: The loop over $\solveridx$ terminates if the
norm of the residual $\Vert\residual_\solveridx\Vert < \max\{\delta_\text{abs},
\delta_\text{rel} \Vert\pseudotargets - \vm\Vert\}$ is below an absolute tolerance
$\delta_\text{abs}$ or below the scaled norm of the right-hand side $\pseudotargets -
\vm$ of the linear system. By default, both tolerances are set to $10^{-5}$.
Additionally, we typically specify a maximum number of iterations. The solver is also
terminated when the normalization constant $\searchdirsqnorm_\solveridx \leq 0$. This
can happen due to numerical imprecision if the linear system is badly conditioned, \eg
if some eigenvalues of the linear system are close to zero.

\subsection{Cost Analysis of \iterncgp}
\label{sec:details_cost_analysis}

In this section, we investigate the computational costs of \iterncgp in more detail. We
start with a discussion of the computational costs for matrix-vector products with
$\mK$, $\mW^{-1}$ and $\mC_\solveridx$ and then analyze the runtime and memory costs of
the individual algorithms (\cref{alg:outer,alg:inner,alg:virtual_solver_run}).

\subsubsection{Matrix-Vector Products}

\itergp is an \textit{iterative matrix-free} algorithm and our algorithm \iterncgp
inherits that property: The matrices $\mK$, $\mW^{-1}$ and $\mC_\solveridx$ are
evaluated lazily, \ie matrix-vector products are evaluated \textit{without} forming the
matrices in memory explicitly. This enables our algorithm to scale to problems where a
naive approach causes memory overflows. In
\cref{alg:outer,alg:inner,alg:virtual_solver_run}, the memory and runtime cost for
matrix-vector products with $\mK$ are denoted by $\spaceK$ and $\timeK$ and by
$\spaceWinv$ and $\timeWinv$ for products with $\mW^{-1}$.

\textbf{Products with $\mK$.} Matrix-vector products with $\mK$ can be decomposed into
products with its sub-matrices. The associated memory costs $\bigO{\spaceK}$ can thereby
be reduced basically arbitrarily and the runtime can be improved by using specialized
software libraries such as \keops \citep{Charlier2021} and parallel hardware (\ie GPUs).
Still, products with $\mK$ are computationally relatively expensive, since this
operation is typically \textit{quadratic} in the number of training \datapoints
$\numtraindata$.

\textbf{Products with $\mW^{-1}$ (General Case).} Under the assumptions on the
likelihood from \cref{sec:ncgps}, $\mW$ is block-diagonal with $\numtraindata$ blocks of
size $\numclasses \times \numclasses$ (in \cnordering, see \cref{sec:ordering}). Here,
we denote these blocks by $\mW_1, \dots, \mW_\numtraindata$. It can be easily verified
that $\mW^{-1}$ is also a block-diagonal matrix and the blocks on its diagonal are the
inverses of $\mW_1, \dots, \mW_\numtraindata$.

Consider the matrix-vector product $\vv \mapsto \mW^{-1} \vv \eqdef \vw \in
\R^{\numtraindata\numclasses}$. In the vectors $\vv$ and $\vw$, we repeatedly group
$\numclasses$ consecutive entries which results in segments $\vw_n, \vv_n \in
\R^{\numclasses}$ for $n = 1, ..., \numtraindata$, \ie
\begin{align*}
  \underbracket[0.1ex]{
  \begin{pmatrix}
    \mW_1^{-1} & & \\
    & \ddots & \\
    & & \mW_\numtraindata^{-1}
  \end{pmatrix}
  }_{= \mW^{-1}}
  \cdot
  \underbracket[0.1ex]{
    \begin{pmatrix}
      \vv_1 \\
      \vdots \\
      \vv_\numtraindata
    \end{pmatrix}
  }_{= \vv}
  =
  \underbracket[0.1ex]{
    \begin{pmatrix}
      \vw_1 \\
      \vdots \\
      \vw_\numtraindata
    \end{pmatrix}.
  }_{= \vw}
\end{align*}
It holds that $\vw_n = \mW_n^{-1} \vv_n$, \ie each segment in $\vw$ is the product of a
single $\numclasses \times \numclasses$ block from $\mW^{-1}$ with one segment from
$\vv$. Computing $\vw_n$ thus amounts to solving a linear system of size $\numclasses$
with cost $\bigO{\numclasses^3}$. The total cost for all $\numtraindata$ segments is
thus $\bigO{\numtraindata\numclasses^3}$. However, the $\numtraindata$ linear systems
are independent of each other and can thus be solved in parallel. So, if appropriate
computational resources are available, the total runtime complexity can be reduced to
$\bigO{\numclasses^3}$.

In general, $\mW^{-1}$ requires $\bigO{\numtraindata \numclasses^2}$ in terms of memory
consumption. If needed, these costs can be reduced further to $\bigO{\numclasses^2}$
because (as explained above), products with $\mW^{-1}$ can be broken down into products
with the individual blocks of $\mW^{-1}$. We can perform those products sequentially
such that only a single block is present in memory at a time.

\textbf{Products with $\mW^{-1}$ (Special Cases).} In many cases, we can multiply with
$\mW^{-1}$ more efficiently. In the multi-class classification case, the runtime and
memory costs for multiplication with the pseudo inverse $\mW^\pinv$ can be reduced to
$\bigO{\numtraindata\numclasses}$ (see \cref{subsec:pseudo_inverse_w}).
In the regression case ($\numclasses = 1$), $\mW^{-1}$ is a diagonal matrix of size
$\numtraindata \times \numtraindata$. The memory and runtime costs are thus in
$\bigO{\numtraindata}$. An example is Poisson regression, for which we derive the
explicit form of $\mW^{-1}$ in \cref{sec:math_poisson_regression}.

\textbf{Products with $\mC_\solveridx$.} $\mC_\solveridx = \matrixroot_\solveridx
\matrixroot_\solveridx^\top$ is represented via its matrix root $\matrixroot_\solveridx
\in \R^{\numtraindata\numclasses \times \buffersize}$. This allows for efficient storage
and matrix-vector multiplies $\vv \mapsto \mC_\solveridx \vv = \matrixroot_\solveridx
(\matrixroot_\solveridx^\top \vv)$ in $\bigO{\buffersize\numtraindata\numclasses}$.

\subsubsection{Cost Analysis \cref{alg:inner,alg:outer,alg:virtual_solver_run}}

The runtime and memory complexity for the operations in
\cref{alg:inner,alg:outer,alg:virtual_solver_run} is given directly in the pseudo code.
Here, we provide some additional information for the costs that depend on the user's
choices and put the costs of the individual algorithms into perspective.

\textbf{\cref{alg:inner} (\itergp).} The runtime cost for selecting an action
$\bigO{\timepolicy}$ depends on the underlying policy. For Cholesky actions
($\action_\solveridx = \ve_\solveridx$) or CG ($\action_\solveridx = \vr_{\solveridx -
1}$), the \runtime cost is insignificant since no additional computations are required
at all.

One iteration's total computational cost (without prediction) is dominated by two
matrix-vector products with $\mK$ in terms of runtime and
$\bigO{\buffersize\numtraindata\numclasses}$ in terms of storage requirements (for the
buffers $\mS$ and $\mT$ as well as the matrix root $\matrixroot_\solveridx$). The
\textit{initial} size (\ie the number of columns) of $\mS$, $\mT$ and
$\matrixroot_\solveridx$ is given by the rank limit $\bufferlimit$ used in
\cref{alg:virtual_solver_run}. Henceforth, one column is added to each of the buffers
and matrix root in \textit{each} solver iteration, increasing their size to $\buffersize
= \bufferlimit + \solveridx$ in iteration $\solveridx$. It is thus reasonable to include
an upper bound on the iteration number in the stopping criterion of \cref{alg:inner}.

\textbf{\cref{alg:virtual_solver_run} (Virtual Solver Run with Optional Compression).}
The total runtime complexity of \cref{alg:virtual_solver_run} is
$\bigO{\buffersize\timeWinv + \buffersize^2 \numtraindata\numclasses}$, \ie dominated by
matrix-matrix products involving the buffers and $\mW^{-1}$. In terms of memory
requirements, the buffers $\mS$, $\mT$, and $\matrixroot_0$ are the decisive
contributors with $\bigO{\bufferlimit\numtraindata\numclasses}$. The truncation of the
eigendecomposition provides a way to control that bound by resetting the current buffer
size $\buffersize$ to an arbitrary number $\bufferlimit \leq \buffersize$. In comparison
to \cref{alg:inner}, the computational costs are practically of minor importance since no
multiplications with $\mK$ are necessary.

\textbf{\cref{alg:outer} (\iterncgp Outer Loop).} The costs $\bigO{\timemean}$ for
evaluating $m$ on the training data depends on the choice of mean function. For a
constant mean function, no computations are necessary, so runtime costs are negligible.
This can be different \eg for applications in Bayesian deep learning, where evaluating
$m$ requires forward passes through a neural network.

\section{Experimental Details}
\label{sec:experimental_details}

Throughout the paper, we use binary classification as an illustrative and supporting
example
(\cref{fig:visual_abstract,fig:pseudo_targets_noise,fig:policy_choice_mean,fig:virtual_solver_run}).
The two main experiments follow in \cref{sec:experiments}: Poisson regression
(\cref{sec:exp_poisson_regression}, \cref{fig:poisson_regression}) and large-scale GP
multi-class classification (\cref{sec:exp_large_scale_gpc},
\cref{fig:acc_loss_ece_over_runtimes}). In the following, we provide additional details
for all those experiments.

\subsection{Binary Classification}
\label{sec:details_binary_classification}

\textbf{Binary Classification with \textit{one} latent function.} Consider a binary
classification task, \ie $C=2$. Being able to report the probability for \textit{one} of
the two classes is sufficient because they have to add up to one for every \datapoint.
Thus, while $C = 2$, $\numtraindata$-dimensional vectors are typically used to describe
this problem \citep[Section 3.4]{Rasmussen2006}. Using only a single latent function is
convenient for illustrative purposes, as \eg the action vectors $\action$ in
\cref{alg:inner} are $\numtraindata$-dimensional (not $2\numtraindata$-dimensional) and
thus easier to visualize.

\textbf{1D Data.} We use a one-dimensional training set in
\cref{fig:pseudo_targets_noise}. $\mX$ is created by sampling $\numtraindata = 50$
\datapoints between $-3$ and $5$. The hidden function $f$ is a draw from a GP with mean
zero and a \gpytorch \citep{Gardner2018GPyTorchBlackbox} RBF kernel with
\texttt{lengthscale = 1.0} and \texttt{outputscale = 5.0}. For each datapoint $\vx_n$,
we sample the positive label with probability $\operatorname{logistic}(f(\vx_n))$.

\textbf{2D Data.} Two-dimensional data is used in
\cref{fig:visual_abstract,fig:policy_choice_mean,fig:virtual_solver_run}. The
data-generating process is analogous to the 1D data, only now, the $\numtraindata = 100$
training inputs are in the 2D plane: The first coordinate is sampled uniformly between
$-3$ and $5$, the second between $-4$ and $1$. The hyperparameters of the RBF kernel are
\texttt{lengthscale = 1.0}, \texttt{outputscale = 10.0} for
\cref{fig:policy_choice_mean,fig:virtual_solver_run} and \texttt{outputscale = 20.0} for
\cref{fig:visual_abstract}.

\textbf{Details \cref{fig:visual_abstract}.} In this figure, we compare two versions of
our algorithm: \iterncgp-Chol without recycling and \iterncgp-CG with recycling and with
compression ($\bufferlimit = 10$). Both runs were conducted on a CPU. The computation of
the NLL loss is \textit{not} included in the runtime measurement. A description of how
the NLL loss can be computed for arbitrary $\numclasses$ is given in
\cref{sec:details_large_scale_gpc}.

\textbf{Details \cref{fig:pseudo_targets_noise}.} For \cref{fig:pseudo_targets_noise},
we compute a sequence of \textit{precise} Newton steps by using \iterncgp with unit
vector actions and $\solveridx \leq \numtraindata$ solver iterations. Note that the
Newton linear system is $\numtraindata$-dimensional, \ie we actually obtain
$\vf_\newtonidx$ as defined by \cref{eq:newton_step}.

\textbf{Details \cref{fig:policy_choice_mean}.} In this plot, we compare unit vector
actions (\iterncgp-Chol) and residual actions (\iterncgp-CG) for the first Newton step
($\newtonidx = 0$) at three solver iterations $\solveridx \in \{1, 10, 19\}$. The true
posterior mean function $m_{0,*}$ and covariance function $\kernelfn_{0,*}$ are computed
by using \iterncgp-Chol and $j \leq \numtraindata = 100$ iterations.
\cref{fig:policy_choice_var} shows the covariance functions corresponding to the mean
functions in \cref{fig:policy_choice_mean}.

\begin{figure*}[tbh!]
  \includegraphics{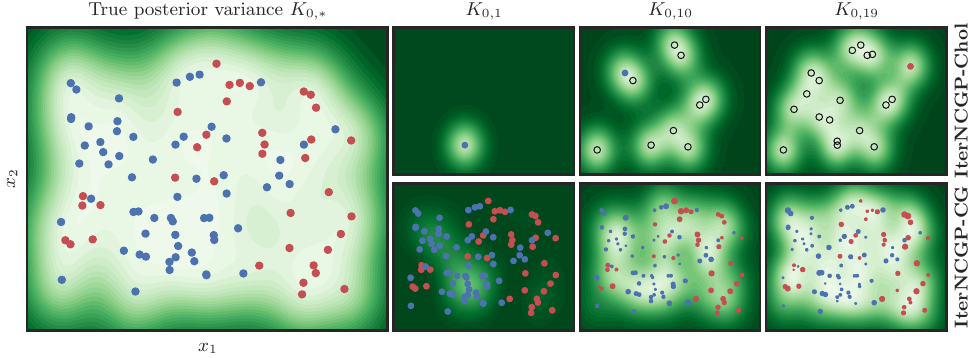}
  \caption{\textbf{Different Policies of \iterncgp Applied to GP Classification.}
  \emph{(Left)} The true posterior covariance $\kernelfn_{0,*}$
  (\,\colorgradientbox[black!70!SNSgreen]{SNSgreen}\,) for a binary classification
  task (\colordot{PlotBlue}/\colordot{PlotRed}). \emph{(Right)} Current estimate of
  the posterior covariance after $1, 10,$ and $19$ iterations with the unit vector
  policy (\emph{Top}) and the CG policy (\emph{Bottom}). Shown are the \datapoints
  selected by the policy in this iteration with the dot size indicating their relative
  weight. For \iterncgp-Chol, \datapoints are targeted one by one and previously used
  \datapoints are marked with (\colorcircle{black}). }
  \label{fig:policy_choice_var}
\end{figure*}

The actions are visualized by scaling the training \datapoints according to their
\textit{relative weight}: First, we take the absolute value of the action vector
$\action$ from \cref{alg:inner} (element-wise) and then scale its entries linearly such
that the smallest entry is $0$ and the largest is $1$.

\textbf{Details \cref{fig:virtual_solver_run}.} In this figure, we show the initial mean
function $m_{1,0}$ and covariance function $\kernelfn_{1,0}$ in the \textit{second}
Newton step for different buffer sizes $\bufferlimit \in \{0, 1, 3, 10\}$. We use CG
actions for the first Newton step and let the solver run until convergence (this takes
$19$ iterations).

\subsection{Poisson Regression}
\label{sec:details_poisson_regression}

In \cref{sec:exp_poisson_regression}, we apply \iterncgp to Poisson regression to
demonstrate our algorithm's ability to generalize to other (log-concave) likelihoods and
to explore the trade-off between the number of (outer loop) mode-finding steps and
(inner loop) solver iterations.

\textbf{Poisson Likelihood.} We consider count data $\vy \in \N_0^\numtraindata$ that is
assumed to be generated from a Poisson likelihood with unknown positive rate $\lambda
\colon \inputspace \to \R_+$. Modeling $\lambda$ with a GP which may take
positive \textit{and negative} values, would therefore be inappropriate. However, we can
use a GP for the log rate $f \colon \inputspace \to \R$ and regard this as the unknown
latent function. With $\vlambda \defeq \lambda(\mX) = \exp(f(\mX)) = \exp(\vf)$, the
likelihood is given by
\begin{equation*}
    p(\vy \mid \vf)
    = \prod_{n=1}^\numtraindata
    \frac{\lambda_n^{y_n} \exp(-\lambda_n)}{y_n!}.
\end{equation*}
The gradient and (inverse) Hessian of the log likelihood can be derived in closed form,
see \cref{sec:math_poisson_regression}.

\textbf{Data \& Model.} First, we create $\mX$ by linearly spacing $\numtraindata=100$
points between $0$ and $1$. For the count data $\vy$, we sample from a GP with zero mean
and a \gpytorch \citep{Gardner2018GPyTorchBlackbox} RBF-kernel with \texttt{lengthscale
= 0.1} and \texttt{outputscale = 5.0}. That GP $f$ represents the log-Poisson rate. We
then draw counts from a Poisson distribution with rate $\lambda(\vx_n) = \exp(f(\vx_n))$
for each \datapoint in the training set. In this experiment, we conduct multiple
\iterncgp runs on different training sets. These sets are created by re-drawing from the
Poisson distributions with the same rates, \ie the underlying GP for the log rate does
\textit{not} change. Our NCGP's prior uses the same RBF kernel to avoid model mismatch.

\textbf{\iterncgp-CG Approaches.} We consider \iterncgp-CG with four different schedules:
A fixed budget of $100$ iterations is distributed uniformly over $5, 10, 20$ or $100$
outer loop steps (see \cref{alg:outer}), which limits the number of inner loop
iterations (see \cref{alg:inner}) to $j \leq 20, 10, 5$ or $1$. For each schedule, we
perform $10$ runs with different training sets, see above. Each run uses recycling
without compression. For this experiment, the convergence tolerance in \cref{alg:outer}
is set to $\convtol = 0.001$. All runs are performed on a single \textsc{NVIDIA GeForce RTX 2080
Ti 12 GB GPU}.

\textbf{Tracking of Performance Metrics.} As a performance metric, we use the NLL loss.
The computation of the NLL loss for the test and training set is \textit{not} included
in the runtime reported in the results. For the NLL loss, we approximate the integral
from \cref{subsec:prediction} with MC samples: For each test datum
$\vx_\symboltestdata$, we draw $10^5$ MC samples from
$\gaussian{m_{\newtonidx,\solveridx}(\vx_\symboltestdata)}{\kernelfn_{\newtonidx,\solveridx}(\vx_\symboltestdata,
\vx_\symboltestdata)}$, map those samples $\{f_{\symboltestdata, k}\}_{k=1}^{10^5}$
through the likelihood $p(y_\symboltestdata \mid f_{\symboltestdata, k})$ and average.
This yields a loss value for $\vx_\symboltestdata$ and we obtain the training/test NLL
loss by averaging these loss values for all \datapoints in the training/test set.

\textbf{Approximate Rate Distribution.} Using \iterncgp-CG for the Poisson regression
problem results in a sequence of posteriors
$\gp{m_{\newtonidx,\solveridx}}{k_{\newtonidx,\solveridx}}$. By drawing MC samples from
those posterior GPs and mapping them through the exponential, we obtain an approximated
(skewed) distribution for the rate $\lambda$. In \cref{fig:poisson_regression}
\emph{(Right)}, we report its median and a $95~\%$ confidence interval between the
$2.5~\%$ and $97.5~\%$ percentile.

\subsection{Large-Scale GP Multi-Class Classification}
\label{sec:details_large_scale_gpc}

In this experiment, we empirically evaluate \iterncgp on a large-scale GP multi-class
classification problem to exhibit its scalability. We also investigate the impact of
compression on performance.

\textbf{Data.}
We consider a Gaussian mixture problem with $C=10$ classes in 3D. For each class, we
sample a mean vector uniformly in $[-1, 1]^3$ and a positive definite covariance matrix.
For the covariance matrix, we first create a $3 \times 3$ matrix $\mC$ with entries
between $0$ and $1$ (sampled uniformly) and compute the eigenvectors $\mU$ of $\mC
\mC^\top$. Then, we create three eigenvalues $\{\lambda_d\}_{d \in \{1, 2, 3\}}$
uniformly between $0.001$ and $0.1$ and form the covariance matrix from the eigenvectors
$\mU$ and these eigenvalues, \ie $\mU \diag(\lambda_1, \lambda_2, \lambda_3)\,\mU^\top$.
For each class, $10^4$ \datapoints are sampled from the respective Gaussian
distribution. This amounts to $\numtraindata = 10^5$ \datapoints in total. For testing,
$\numtraindata_\symboltestdata = 10^4$ \datapoints are used ($10^3$ per class).

Our benchmark (\cref{fig:acc_loss_ece_over_runtimes}) uses multiple runs for each
method. The runs differ in the seed that is used to sample from the Gaussians, \ie the
training and test set are different for each run (the underlying Gaussians distributions
remain the same). Both the training and test set used in the first run are shown in
\cref{fig:training_test_data}.

\begin{figure*}[tbh!]
  \includegraphics[width=0.99\linewidth]{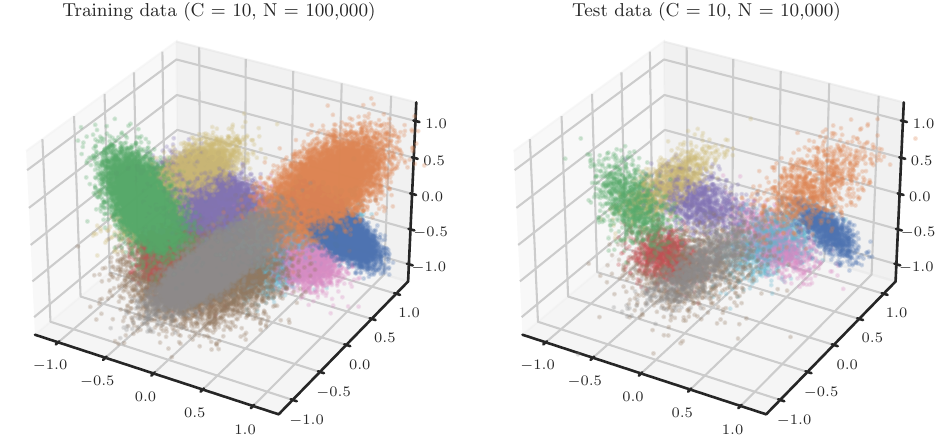}
  \caption{\textbf{Gaussian Mixture Training and Test Data.}
        Training data \emph{(Left)} and test data \emph{(Right)} for the first run. Note
        that the underlying Gaussians are the same for training and test set and for all
        runs.}
  \label{fig:training_test_data}
\end{figure*}

\textbf{Model.} We use a softmax likelihood (see \cref{subsec:pseudo_inverse_w} for the
details on the pseudo inverse $\mW^\pinv$) and assume independent GPs for the
$\numclasses$ outputs of the latent function. Each of these GPs uses the zero function
as the prior mean and a \matern{3} kernel. We use the \keops \citep{Charlier2021}
version of the \gpytorch \citep{Gardner2018GPyTorchBlackbox} kernel with
\texttt{lengthscale = 0.05} and \texttt{outputscale = 0.05}.

\textbf{SoD Approaches.} For the SoD approaches, we create a random subset of the
training data (sampling without replacement) of a specific subset size
$\numtraindata_\text{sub}$. We then explicitly form $\KWinvf{\vf_i}$ for every Newton
step and compute its Cholesky decomposition via \pytorch{}'s \citep{Paszke2019PyTorch}
\texttt{torch.linalg.cholesky} (instead of using \itergp in \cref{alg:outer} to ensure a
competitive baseline implementation). In our experiment, we use four different subset
sizes $\numtraindata_\text{sub} \in \{250, 500, 1000, 2000\}$. Each setting is run five
times with different random seeds (see above).

\textbf{\svgp.} \svgp \citep{Titsias2009,Hensman2015ScalableVariational} is a commonly
used variational method for approximative inference in non-conjugate GPs. We use
\gpytorch's \citep{Gardner2018GPyTorchBlackbox} \svgp implementation and optimize the
ELBO for $10^4~\text{seconds}$ using \adam  with batch size $ = 1024$. The learning
rate $\alpha \in \{0.001, 0.01, 0.05\}$ and the number of inducing points
$U \in \{1000, 2500, 5000, 10000\}$ are tuned via grid search (using only a single
run). We use $\sfrac{U}{C}$ inducing points per class ($C = 10$ classes) and
initialize them as a random subset of the training data. Within the given runtime
budget, \svgp performs between $6000$ ($U = 1000$) and $600$ ($U = 10000$) epochs.

For each of the $12$ runs, we extract $6$ performance indicators: lowest training/test
NLL loss during training, highest training/test accuracy during training and
training/test expected calibration error (ECE) at the very end of training. For each of
those $6$ categories, we determine those two runs with the best performance. This
procedure results in a set of $3$ runs in total, that are considered the ``best'' runs
for \svgp. Only for those $3$ settings, we perform $5$ runs each and report their
mean performance in \Cref{fig:acc_loss_ece_over_runtimes}.

\textbf{\iterncgp-CG Approaches.} For comparison, we apply our matrix-free algorithm
\iterncgp with residual actions to the \textit{full} training set. We use two
configurations: The first one uses recycling \textit{without} compression (\ie
$\bufferlimit = \infty$), the second one uses recycling \textit{with} compression
($\bufferlimit = 10$). The number of solver iterations per step is limited by
$\solveridx \leq 5$. For both settings, we perform $5$ runs with different random seeds
(see above) and report their mean performance in \Cref{fig:acc_loss_ece_over_runtimes}.

\textbf{Tracking of Performance Metrics.} As performance metrics, we use accuracy, the
negative log likelihood (NLL) loss and the expected calibration error (ECE)
\citep{Kumar2019VerifiedCalibration} on both the training and test set. The computation
of those six metrics is \textit{not} included in the runtime reported in the results
(\Cref{fig:acc_loss_ece_over_runtimes}).
First, we compute the predictive mean $m_{\newtonidx, \solveridx}(\vx_\symboltestdata)$
and marginal variance $\diag(\kernelfn_{\newtonidx, \solveridx}(\vx_\symboltestdata,
\vx_\symboltestdata))$ (see
\cref{eq:inference_PLS_iteration_mean,eq:inference_PLS_iteration_covar}) for each test
input $\vx_\symboltestdata$. Then, we use the probit approximation
\citep{MacKay1992EvidenceFramework,Spiegelhalter1990SequentialUpdating} to obtain the
predictive probabilities
\begin{equation*}
  \vpi_\symboltestdata = \operatorname{softmax}
  \left(
    \frac{m_{\newtonidx, \solveridx}(\vx_\symboltestdata)}{
      \sqrt{1 + \frac{\pi}{8}
      \diag(\kernelfn_{\newtonidx, \solveridx}
      (\vx_\symboltestdata, \vx_\symboltestdata))}
    }
  \right)
  \in \R^\numclasses,
\end{equation*}
where the vector division is defined element-wise. This is an approximation of the
integral from \cref{subsec:prediction}. All three performance metrics are based on the
predictive probabilities.
\begin{itemize}
  \item \textbf{Accuracy.} The prediction for $\vx_\symboltestdata$ is given by
  $\argmax_c([\vpi_\symboltestdata]_c)$, \ie by the class with the largest predictive
  probability. The accuracy is defined as the ratio of correctly classified data.

  \item \textbf{NLL Loss.} The NLL loss for $\vx_\symboltestdata$ is defined as the
  log-probability for the actual class $y_\symboltestdata$, \ie
  $\log([\vpi_\symboltestdata]_{y_\symboltestdata})$. We obtain the NLL training and
  test loss by averaging the individual loss values for the entire training/test set.

  \item \textbf{ECE.} The expected calibration error (ECE)
  \citep{Kumar2019VerifiedCalibration} is a measure for the calibration of the
  predictive probabilities. It groups the probabilities of the predicted classes (\ie
  the classification confidences) into bins and, within these bins, compares the
  average confidence with the actual accuracy. We use
  \texttt{MulticlassCalibrationError} from \texttt{torchmetrics}
  \citep{detlefsen2022torchmetrics} with default parameter \texttt{n\_bins=15}.
\end{itemize}

\textbf{Individual runs.} \Cref{fig:acc_loss_ece_over_runtimes} shows the
\textit{average} performance for each of the nine methods/variants over five runs. In
order to show, that the observed performance differences are not due to chance, we show
the \textit{individual} runs in \cref{fig:acc_loss_ece_over_runtimes_all_runs}.

\begin{figure}[tbh!]
  \includegraphics[width=0.99\linewidth]{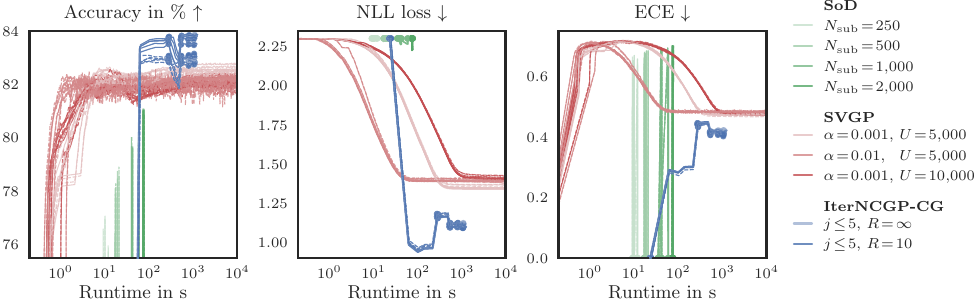}
  \caption{\textbf{Large-Scale GP Classification.}
    Same as \cref{fig:acc_loss_ece_over_runtimes}, but showing all runs
    individually (instead of just their average).}
  \label{fig:acc_loss_ece_over_runtimes_all_runs}
  \vspace{-1ex}
\end{figure}

\subsection{GP Multi-Class Classification on \mnist}
\label{sec:details_gpc_mnist}

To demonstrate \iterncgp{}'s applicability to real-world datasets, we perform an
additional experiment (similar to the experiment described in
\cref{sec:details_large_scale_gpc}) on a subset
($\numtraindata_{\text{sub}} = \num[group-separator={,}]{20000}$) of the \mnist
\citep{Lecun1998Gradient} dataset. In the following, we describe the experiment and
results in more detail.

\textbf{Remark.} For the experiment on synthetic data, we use \keops on top of \gpytorch
for fast kernel-matrix multiplies. However, \keops scales poorly with the data dimension
($\inputdim = 28^2 = 784$ for \mnist)\footnote{ \url{
https://www.kernel-operations.io/keops/\_auto\_benchmarks/plot\_benchmark\_high\_dimension.html
} (accessed May 2024)}. We therefore revert to \gpytorch{}'s standard implementation.
This implementation of the kernel matrix-vector product is fast but causes out-of-memory
errors for large datasets, which is why we limit our benchmark to
$\numtraindata_{\text{sub}} = \num[group-separator={,}]{20000}$ training data. To fully
realize the potential of our method using \keops, it might be advisable to apply it only
to problems with data dimensions smaller than around $100$.

\textbf{Data \& Model.} We use $\num[group-separator={,}]{20000}$ training and
$\num[group-separator={,}]{10000}$ test images from the \mnist dataset and the softmax
likelihood. Our model for the latent function is a multi-output GP which uses
$\numclasses = 10$ independent GPs, each of which is equipped with a \matern{3} kernel.

\textbf{Kernel Hyperparameters.} As a first step, we determine suitable hyperparameters
(the outputscale and lengthscale) for the \matern{3} kernel by running \gpytorch{}'s
\svgp implementation. We use $\num[group-separator={,}]{1000}$ inducing points per
class, a batch size of $1024$ and optimize the ELBO using \adam with learning rate
$0.001$ for $30$ epochs (this results in \texttt{lengthscale = 1.550934} and
\texttt{outputscale = 0.451591}). Note that choosing hyperparameters with \svgp may give
\svgp an advantage in what performance it can reach, making it a competitive baseline.

\textbf{\svgp and \iterncgp-CG Approaches.} We compare \iterncgp-CG and \svgp both using
the same fixed kernel hyperparameters (see above). \iterncgp-CG is applied with recycling
and $R  =  \infty$, \ie without compression. We exclude compression since both \iterncgp
runs converge within three iterations, see below. The number of inner loop iterations is
limited by $\solveridx  \leq  1$ or $\solveridx  \leq  5$, \ie two runs are performed. For
the \svgp approach, we optimize the ELBO using \adam with batch size $1024$ for $200$
seconds. The number of inducing points and the learning rate are tuned via grid search.
As in \cref{sec:details_large_scale_gpc}, we use $U  \in  \{1000, 2500, 5000, 10000\}$
inducing points and three different learning rates $\alpha  \in  \{0.001, 0.01, 0.05\}$.
All $14$ runs are performed on a single \textsc{NVIDIA GeForce RTX 2080 Ti 12 GB GPU}.

\textbf{Results.} The results are shown in \cref{fig:mnist_results}. They show the two
\iterncgp runs and the best four \svgp runs (these include the best two runs for each of
the six performance metrics training/test accuracy/NLL/ECE). Our observations are mostly
consistent with the results on the synthetic data
(\Cref{fig:acc_loss_ece_over_runtimes}): Both \iterncgp runs significantly outperform
the best \svgp runs in terms of NLL loss and accuracy. Only the ECE achieved by
\iterncgp is slightly worse than for \svgp. However, as explained in
\cref{sec:exp_large_scale_gpc}, a small ECE on its own is not conclusive since we can
easily construct a classifier with perfect ECE by randomly sampling predictions.

\stopcontents[sections]

\end{document}